\documentclass[11pt]{article}

\usepackage{algorithm, algorithmic}
\usepackage{amsmath, amssymb, amsthm}
\usepackage{framed}
\usepackage{fullpage}
\usepackage[margin=1.1in]{geometry}
\usepackage{graphicx}
\usepackage{verbatim}
\usepackage{color}
\usepackage{kpfonts}
\definecolor{DarkGreen}{rgb}{0.1,0.5,0.1}
\definecolor{DarkRed}{rgb}{0.5,0.1,0.1}
\definecolor{DarkBlue}{rgb}{0.1,0.1,0.5}
\usepackage{enumitem}
\usepackage[pdftex]{hyperref}
\hypersetup{
    unicode=false,				
    pdftoolbar=true,				
    pdfmenubar=true,			
    pdffitwindow=false, 		
    pdftitle={},    					
    pdfauthor={}
    pdfsubject={},				
    pdfnewwindow=true,		
    pdfkeywords={},			
    colorlinks=true,				
    linkcolor=DarkBlue,		
    citecolor=DarkGreen,	
    filecolor=DarkRed,		
    urlcolor=DarkBlue,		
}

\newcommand\N{\mathbb{N}}

\newcommand{\cA}{\mathcal{A}}

\newcommand{\cD}{\mathcal{D}}
\newcommand{\cE}{\mathcal{E}}

\newcommand{\cO}{\mathcal{O}}

\newcommand{\poly}{\mathrm{poly}}

\newcommand{\bits}{\{0,1\}}
\newcommand{\getsr}{\gets_{\mbox{\tiny R}}}
 
\newcommand{\set}[1]{\left\{#1\right\}} 
\newcommand{\from}{:}
\newcommand{\negl}{\mathrm{negl}}

\newcommand{\ex}[1]{\mathbb{E}\left[#1\right]}
\DeclareMathOperator*{\Expectation}{\mathbb{E}}
\DeclareMathOperator*{\E}{\mathbb{E}}
\newcommand{\Ex}[2]{\Expectation_{#1}\left[#2\right]}
\DeclareMathOperator*{\Probability}{\mathrm{Pr}}
\newcommand{\prob}[1]{\mathrm{Pr}\left[#1\right]}
\newcommand{\Prob}[2]{\Probability_{#1}\left[#2\right]}

\newcommand{\INDSTATE}[1][1]{\STATE\hspace{#1\algorithmicindent}}

\newcommand{\cryptoalg}[1]{\mathit{#1}}
\newcommand{\cryptoadv}[1]{\mathcal{#1}}
\newcommand{\gen}{\cryptoalg{Gen}}
\newcommand{\enc}{\cryptoalg{Enc}}
\newcommand{\dec}{\cryptoalg{Dec}}
\newcommand{\trace}{\cryptoalg{Trace}}
\newcommand{\encscheme}{\cryptoalg{Enc}}
\newcommand{\encgen}{\cryptoalg{\gen}}
\newcommand{\encenc}{\cryptoalg{\enc}}
\newcommand{\encdec}{\cryptoalg{\dec}}

\newcommand{\fpcscheme}{\cryptoalg{FPC}}
\newcommand{\fpcgen}{\cryptoalg{\fpcscheme.\gen}}
\newcommand{\fpctrace}{\cryptoalg{\fpcscheme.\trace}}
\newcommand{\security}{\lambda}

\newcommand{\oracle}{\cO}
\newcommand{\cryptogame}[1]{\mathsf{#1}}
\newcommand{\realgame}{\cryptogame{Attack}}
\newcommand{\idealgame}{\cryptogame{IdealAttack}}
\newcommand{\accuracygame}{\cryptogame{Acc}}
\newcommand{\CPAgame}{\cryptogame{CPAGame}}
\newcommand{\bnpgame}{\cryptogame{NonPrivacy}}
\newcommand{\bnpattack}{\cryptogame{PrivacyAttack}}
\newcommand{\idealbnpattack}{\cryptogame{IdealPrivacyAttack}}
\newcommand{\encadv}{\cryptoadv{B}}
\newcommand{\fpcadv}{\cryptoadv{A_{\fpcscheme}}}
\newcommand{\accadv}{\cryptoadv{A}}
\newcommand{\bnpadv}{\cryptoadv{A}_{\mathrm{priv}}}
\newcommand{\consistent}{\mathrm{Con}}
\newcommand{\sk}{sk}
\newcommand{\ssk}{\overline{\sk}}
\newcommand{\keylength}{\ell_{\encscheme}}
\newcommand{\encoracle}{\cE}
\newcommand{\ct}{c}

\newcommand{\pop}{p}
\newcommand{\sample}{n}
\newcommand{\length}{\ell_{\fpcscheme}}
\newcommand{\dimension}{d}
\newcommand{\dist}{\cD}
\newcommand{\queries}{k}
\newcommand{\query}{q}

\newtheorem{theorem}{Theorem}[section]
\newtheorem{lemma}[theorem]{Lemma}

\newtheorem{claim}[theorem]{Claim}

\theoremstyle{definition}
\newtheorem{definition}[theorem]{Definition}

\def\ShowAuthNotes{0}
\ifnum\ShowAuthNotes=1
\newcommand{\authnote}[2]{{ \footnotesize \bf{\color{DarkRed}[#1's Note:
{\color{DarkBlue}#2}]}}}
\else
\newcommand{\authnote}[2]{}
\fi

\newcommand{\jnote}[1]{{\authnote{Jon} {#1}}}

\title{Preventing False Discovery in \\ Interactive Data Analysis is Hard}
\author{Moritz Hardt\thanks{IBM Research Almaden.  Email: \href{mhardt@us.ibm.com}{mhardt@us.ibm.com}} \and Jonathan Ullman\thanks{Harvard University School of Engineering and Applied Sciences and Center for Research on Computation and Society.  Supported by NSF grant CNS-1237235.  Email: \href{mailto:jullman@seas.harvard.edu}{jullman@seas.harvard.edu}}}

\begin{document}
\maketitle

\begin{abstract}
We show that, under a standard hardness assumption, there is no computationally
efficient algorithm that given $n$ samples from an unknown distribution can give
valid answers to $n^{3+o(1)}$ adaptively chosen statistical queries. A
statistical query asks for the expectation of a predicate over the underlying
distribution, and an answer to a statistical query is valid if it is ``close'' to
the correct expectation over the distribution.

Our result stands in stark contrast to the well known fact that exponentially
many statistical queries can be answered validly and efficiently if the
queries are chosen non-adaptively (no query may depend on the answers to
previous queries). Moreover, a recent work~\cite{DworkFHPRR14} shows how to
accurately answer exponentially many adaptively chosen statistical queries via
a computationally inefficient algorithm; and how to answer a quadratic number
of adaptive queries via a computationally efficient algorithm. The latter
result implies that our result is tight up to a linear factor in~$n.$

Conceptually, our result demonstrates that achieving statistical validity alone can be a source of computational intractability in adaptive settings.  For example, in the modern large collaborative research environment, data analysts typically choose a particular approach based on previous findings. False discovery occurs if a research finding is supported by the data but not by the underlying distribution. While the study of preventing false discovery in Statistics is decades old, to the best of our knowledge our result is the first to demonstrate a computational barrier. In particular, our result suggests that the perceived difficulty of preventing false discovery in today's collaborative research environment may be inherent. 
\end{abstract}

\vfill
\pagebreak

\tableofcontents

\vfill
\pagebreak

\section{Introduction} 
Empirical research is commonly done by testing multiple hypotheses on a finite
sample. A test outcome is deemed statistically signficant if it is unliked to have
occured by chance alone. False discovery arises if the analyst incorrectly declares 
an observation as statistically significant. For decades statisticians have been devising
methods for preventing false discovery, such as the widely used and highly
influential method for controlling the \emph{false discovery rate} due to
Benjamini and Hochberg~\cite{BenjaminiH79}.

Nevertheless the problem of false discovery persists across all empirical
sciences today. Popular articles report on an increasing number of invalid research findings.  
Why is it seemingly so difficult to prevent false discovery?
Today's practice of data analysis diverges from classical statistics in its massive scale,
heavy use of sophisticated algorithms, and large number of
participants in any given project. Importantly, the way modern data analysts
interact with the data set is inherently \emph{adaptive}---many design choices, including
the choice and tuning of the algorithm itself, depend on previous interactions with the
data set. An extreme example are data science competitions, in which hundreds of data scientists
analyze the same data set and repeatedly evaluate their approach against the
same data. This level of adaptivity makes it nearly impossible to give a
precise a priori description of the experimental setup. 

We formalize the way in which data analysts may
interact with a data set using the \emph{statistical-query
model (SQ model)} of Kearns~\cite{Kearns93}. In the statistical query model, there is an
algorithm called the \emph{oracle} that gets access to $n$ samples from an unknown
distribution ${\cal D}$ over some finite universe ${\cal X}.$ We will assume
throughout that ${\cal X}=\{0,1\}^d$ where we think of the parameter~$d$ as
the dimensionality of the data. A statistical
query $q$ is specified by a predicate $p\colon {\cal X}\to \{0,1\}$ and the
answer to a statistical query is defined as $q({\cal D}) = \E_{x\sim{\cal D}}
p(x).$ The goal of the oracle is to give an answer~$a$ to the query that is
\emph{accurate} (or \emph{statistically valid}) in the sense that $|a-q({\cal
D})| \le \alpha$ with high probability. Throughout our work we only require $\alpha$ to
be a small constant. Put differently, the goal of the oracle is to provide
answers that generalize to the underlying distribution rather than answers
that are specific to the sample. The latter is always easy to achieve by
outputting the empirical average of the query predicate on the sample.

The SQ model has a number of advantages for our purposes. First, almost all
natural machine learning algorithms can be compiled into a sequence of
statistical queries. Hence, the model does not give up much generality.
Second, it makes it convenient to formalize adaptivity. In the
adaptive/interactive setting, the analyst is modeled as an efficient algorithm
that given a sequence of queries and answers $q_1,a_1,q_2,a_2,\dots,q_i,a_i$
(previously exchanged with the
oracle) produces a new query~$q_{i+1}$.  
We say that an oracle is \emph{accurate} given $n$ samples for $k$ adaptively chosen
queries, if for every distribution ${\cal D},$ given $n$ samples from ${\cal
D}$ the oracle accurately responds to any computationally efficient adaptive
analyst that makes at most $k$ queries. A computationally efficient oracle
should run time polynomial in $n$ and $d$ on input of each query.

A recent work by Dwork, Feldman, Hardt, Pitassi, Reingold and Roth~\cite{DworkFHPRR14}
addresses the problem of answering adaptive statistical queries. Their main
result implies that there is a computationally inefficient oracle that
accurately answers even an exponential number of adaptively chosen statistical
queries. Moreover, they show that a quadratic number of queries can be
answered accurately and efficiently. Our main theorem shows that these results
are essentially as far as it goes. Under a standard cryptographic hardness
assumption, we show that there is no efficient oracle that is accurate on more
than a cubic number of adaptively chosen queries. 

\vfill


\begin{theorem} \label{thm:main1}
Assuming one-way functions exist, there is no computationally efficient oracle
that given $n$ samples is accurate on $n^{3+o(1)}$ 
adaptively chosen queries.
\end{theorem}

An intuitive interpretation of the theorem is that if an efficient oracle attempts to
answer more than $n^3$ statistical queries it cannot in general
maintain that its answers are statistically valid with respect to the
underlying distribution. Of course, the oracle can always report the exact
answer of the query on its data set. However, this strategy does not maintain
accuracy on adaptive queries in general and---as our theorem shows---neither
does any other computationally efficient approach.  From a technical
perspective our result gives a strong computational lower bound in the
statistical query model. Lower bounds in the statistical query model have been
studied for more than two decades. But more broadly speaking, we interpret our
result as pointing at an inherent computational obstruction to preventing
false discovery in collaborative science.

Note that Theorem~\ref{thm:main1} stands in sharp contrast to the non-adaptive setting. If we fix
queries $q_1,\dots,q_k$ and then sample $n$ items from the
distribution~${\cal D},$ the observed empirical answer to each query on the
data set will be close to the correct answer with high probability so long as
$k = 2^{o(n)}$.  This guarantee follows from a Hoeffding bound together with
the union bound.

Our hardness result applies when the dimensionality of the data grows with the
sample size more than logarithmically so that $2^d$ is no longer polynomial
in~$n.$\footnote{This is under the
stronger but standard assumption that exponentially hard one-way-functions
exist.} 
\jnote{There is a simpler way to show $n \gg 2^d$ is inherent, I wanted to point it out since it will be easier for most readers to see.  OK?}
This requirement is rather mild, and is also necessary.  If $n \gg 2^d$ then the empirical distribution of the $n$ samples will be close to the underlying distribution in statistical distance, and thus every statistical query can be answered accurately given the sample.  More generally, as we discuss in Section~\ref{sec:dp}, there are algorithms that run in time polynomial in
$n$ and $2^d$ and provide accuracy even on an exponential number of adaptively chosen queries~\cite{DworkFHPRR14}.
Thus, our results show that the dimensionality of the data has a major effect on
the hardness of the problem. In fact, we provide a second theorem that shows
that if the dimensionality is polynomially large in~$n,$ 
then we cannot even hope for a \emph{computationally unbounded} oracle that
provides accuracy on adaptive queries.

\begin{theorem} \label{thm:main2}
There is no computationally unbounded oracle that given $n$ samples of dimension
$d= n^{3+o(1)}$ is accurate on $n^{3+o(1)}$ adaptively chosen queries.
\end{theorem}

While the dimension in the previous theorem has to be large, there are
important data sets that exhibit this trade-off between sample size and
dimension. A good example are genome wide association studies (GWAS). Here,
the sample size corresponds to patients with a certain (possibly rare disease)
and is often in the hundreds. The dimensionality corresponds
to the number of relevent positions in the human genome and is often in the
millions. Moreover, the genome resolution is increasing rapidly with new
technology whereas the number of available patients is not.

\paragraph{Conclusion.}
To conclude this discussion of our results, we believe that adaptivity is an
essential element of modern data analysis that ought to be taken into
account by theoretical models. At the same time, our theorems demonstrate the
intrinsic difficulty of coping with adaptivity.

\subsection{Proof overview} \label{sec:proofoverview}

The intuition for our proof is rather simple. We will design a challenge
distribution~$\cD$ and a computationally
efficient adaptive analyst~${\cal A}$ so that the following is true. If any
compuationally efficient oracle~${\cal O}$ is given $n$ samples $S=\{x_1,\dots,
x_n\}$ drawn from ${\cal D}$ then our adaptive analyst~${\cal A}$ is able to 
reconstruct $n' = n-O(1)$ samples $\{y_1,\dots,y_{n'}\}\subseteq S.$ In other
words, the analyst is able to find all but a constant number of samples that
the oracle is using. While the analyst has a priori information about the
distribution ${\cal D}$ it has no information whatsoever about which sample
${\cal O}$ received. Nevertheless, the analyst can reconstruct essentially all
of the hidden sample. Quantitatively, the analyst proceeds in $n-O(1)$ rounds
and each round consists of roughly $n^2$ queries. In each round the analyst
successfully recovers one data item from the oracle provided that the oracle
continues to give accurate answers. After the analyst has recovered almost all
samples, the effective sample size of the oracle has shrunk down to a constant
size. At this point it is easy for the analyst to find queries on which the
oracle gives blatantly inaccurate answers. 

The first problem is to recover even a single data point inside the
oracle's sample. To solve this problem we rely on a cryptographic primitive known as a
fingerprinting code. Fingerprinting codes were introduced by Boneh and
Shaw~\cite{BonehS98} for the problem of watermarking digital content. A
fingerprinting code has two components. The first component generates a set of
``challenge queries.'' The second component is a ``tracing algorithm'' which takes
answers to these queries and returns a data item. The fingerprinting code
gives the guarantee that if the challenge queries are answered accurately, and
by looking only at how each challenge query is defined on~$S,$ then the tracing algorithm
will successfully recover one element in~$S.$ Unfortunately, in general nothing prevents
the oracle from evaluating the queries at points outside of~$S.$ In fact,
information-theoretically the challenge queries used in our attack reveal information about the unknown 
distribution~$\cD$ that the oracle didn't have previously.  
Evaluating the query outside the sample $S$ is somewhat unnatural. For
example, if the oracle simply outputs an empirical quantity that only depends
on the sample this situation will not arise. For such \emph{natural} oracles
our proof is somewhat easier and does not require any cryptographic
assumptions. We therefore present this illustrative special case in
Section~\ref{sec:natural}.

To obtain a result for all computationally bounded oracles, we need to hide
from the oracle the additional information that's revealed by the query
definition outside the sample.  To do so, we use an encryption
scheme to effectively hide the definition of the query on points outside
of~$S$ from the oracle. The encryption is sufficient to show that, assuming
that the oracle is computationally bounded, the tracing algorithm of the
fingerprinting code will succeed. We note that encryption schemes suitable for
our purpose exist under the standard assumption that one-way functions exist.
With this one-round approach in mind, we can proceed iteratively. In the next
round we exclude the previously learned data item from the definition of the
challenge queries, which ensures that the analyst learns a new item in each
round.

There is one important subtlety. The tracing algorithm of the fingerprinting
code will only succeed if the oracle answers the challenge queries accurately
with respect to its sample~$S.$ However, our assumption is that the oracle is
accurate with respect to the underlying distribution~$\cD$ rather than the
sample~$S.$ We need to worry that eventually the sample and the distribution
disagree on the challenge queries.  In this case the oracle may be inaccurate
on its sample (and hence tracing fails), yet still accurate on the
distribution.  To rule out this pathological situation we use a measure
concentration property of our specific choice of fingerprinting code.
Specifically, we the fact that the challenge queries of the code are
essentially random predicates with a certain bias. This property allows us to
use the randomness of the challenge queries to argue that the sample~$S$
approximately agrees with the distribution~$\cD$ on these queries with
sufficiently high probability so long as there are at least $O(1)$ elements in
the sample that we haven't reconstructed yet. Due to the approximation error
incurred here, we also need to use a somewhat stronger primitive called a
robust fingerprinting code that was just recently provided in work by Bun,
Ullman and Vadhan~\cite{BunUV14}, which also satisfies the necessary measure
concentration property.

\subsection{Connection to privacy and reconstruction attacks}
\label{sec:dp}

Our work builds on a close connection to the problem of designing
privacy-preserving oracles. Here, the goal is to provide answers to
statistical queries in such a way that the analyst does not learn the
specifics of individual data records but rather global properties of the
underlying distribution. A successful approach for formalizing this
desideratum is the notion of differential privacy~\cite{DworkMNS06}.
Differential privacy requires that the answers given by the oracle are
randomized in such a way that the presence or absence of any single data item
in the sample cannot be detected. It is known that differential privacy
prevents so-called \emph{reconstruction attacks}. A reconstruction attack is
an algorithm that is able to reconstruct most entries of a data set by
interacting with the oracle. Such an attack demonstrates that the oracle is
\emph{blatantly non-private} (it fails to satisfy not only differential privacy, but any reasonable notion of privacy). Our work can be considered an efficient
reconstruction attack as we give an efficient adaptive analyst that
reconstructs almost all of the data points that the oracle uses if the oracle
provides accuracy on $n^{3+o(1)}$ queries. An immediate consequence of our
work is therefore the following result.

\begin{theorem} \label{thm:main3}
Assuming one-way functions exist, any computationally efficient oracle
that given $n$ samples is accurate on $n^{3+o(1)}$ 
adaptively chosen queries must be blatantly non-private.
\end{theorem} 

This result should be compared with recent work of Ullman~\cite{Ullman13}, which
showed that oracles satisfying differential privacy cannot answer even $n^{2+o(1)}$ non-adaptively chosen 
queries.  Here we show that if the queries are chosen adaptively, then the same conclusion holds even for oracles that merely thwart blatant non-privacy, up to a factor of $n$ loss in the number of queries.

An important difference to the privacy setting is how \emph{accuracy} is
defined. In the privacy setting, accuracy is defined with respect to the
oracle's sample.  It is trivial to maintain accuracy with respect to the sample by answering each query with the sample mean, which succeeds even when the oracle is blatantly non-private.  In the setting of false discovery, we define accuracy with respect to the underlying distribution and show that achieving this notion of accuracy is hard for the oracle.

\paragraph{Upper bounds for answering adaptive queries.}
Differential privacy is also useful in establishing upper bounds in our
setting. At a high-level, differential privacy is a \emph{stability condition}
on an algorithm requiring that the output varies only slightly with the
addition or deletion of a sample point. On the other hand, it is known
that stability implies generalization~\cite{BousquetE02}. Hence, we can think
of the interaction between an oracle and an analyst as a single algorithm that
satisfies a stability guarantee strong enough to imply generalization bounds
with respect to the underlying distribution. This approach was formalized by
Dwork et al.~\cite{DworkFHPRR14} leading to upper bounds in the adaptive
setting when combined with algorithms from the differential privacy
literature. Specifically, work of Roth-Roughgarden~\cite{RothR10} and
Hardt-Rothblum~\cite{HardtR10} addresses differential privacy in the
interactive setting. The latter work shows that $2^{\tilde \Omega(n/\sqrt{d})}$
statistical queries can be answered with constant error under differential
privacy. However, the running time is exponential in~$d.$
Using the results of~\cite{DworkFHPRR14} this leads to the same upper bound in
the adaptive statistical query setting. 
Similarly, there is an efficient differentially private mechanism that gives constant
accuracy for $\tilde{\Omega}(n^2)$ queries. This leads to a computationally efficient
upper bound in our setting. To summarize we state the following theorem.
\begin{theorem}[\cite{DworkFHPRR14}]
There is an inefficient algorithm that accurately answers $2^{\tilde{\Omega}(n/\sqrt{d})}$ adaptively chosen
statistical queries. Moreover, there is an efficient
algorithm that accurately answers $\tilde{\Omega}(n^2)$ adaptively chosen queries.
\end{theorem}

\jnote{New paragraph.  OK?}
We emphasize that exponential running time was known to be inherent for differentially private algorithms that answer $n^{2+o(1)}$ statistical queries~\cite{Ullman13}, but prior to our results it was possible that there was an efficient oracle that accurately answered exponentially many adaptively chosen statistical queries via a different approach.

\subsection{Related work}
The combination of fingerprinting codes and encryption in our one-round approach is a common technique in the construction of ``traitor-tracing schemes.''  Traitor-tracing schemes were introduced by Chor, Fiat, and Naor~\cite{ChorFN94}, also for the problem of secure distribution of digital content.  Dwork et al.~\cite{DworkNRRV09} were the first to show that traitor-tracing schemes can be used to prove computational hardness results for differential privacy.  Ullman~\cite{Ullman13} showed that traitor-tracing schemes with certain non-standard security properties can be used to prove strong computational hardness results for differential privacy, and showed how to construct such a scheme.  In fact, the one-round approach described above closely mirrors the traitor-tracing scheme constructed in~\cite{Ullman13}.  See~\cite{Ullman13} for a more detailed discussion of prior work on traitor-tracing and the issues that arise when using traitor-tracing schemes in the context of differential privacy.  

Our work was also inspired by recent work of Hardt and Woodruff~\cite{HardtW13}, which showed that no low-dimensional linear sketch can give valid answers to even a polynomial number of adaptively chosen queries.  Technically our results are largely orthogonal to theirs, since we consider arbitrary computationally efficient statistical query oracles, rather than linear sketches.  However, their work also noted the connection between differential privacy and validly answering adaptively chosen queries.  On the technical side, our iterative approach was inspired by their results.

\jnote{Let me know if you want me to add anything here.  I just wanted to acknowledge somewhere that we are not totally unaware that hardness-of-learning is a thing people have looked at.}
There is also a large body of work on the computational hardness of certain learning problems.  Many of these 
results have a similar flavor to ours in showing that any computationally efficient algorithm requires either large 
running time or a large number of samples from the distribution in order to learn a valid hypothesis.  However, we
are not aware of any result showing a hardness result specific to adaptively chosen queries.

\subsection*{Acknowledgments}

We are extremely grateful to Aaron Roth for raising the issue of adaptivity in
false discovery control at the Simons Workshop on Differential Privacy.  We
thank Cynthia Dwork and Omer Reingold for introducing us to the area of False
Discovery Control.  We also thank Salil Vadhan for helpful discussions.  We
acknowledge the Simons Institute for Theoretical Computer Science at Berkeley
where this work started.

\section{Preliminaries}

Let $\dist$ be a distribution over $\bits^{\dimension}$, for some parameter
$\dimension \in \N$.  We are interested in answering \emph{statistical
queries} about the distribution $\dist$.  A statistical query on
$\bits^{\dimension}$ is specified by a predicate $\query \from
\bits^{\dimension} \to \bits$ and is defined to be
$$
\query(\dist) = \Ex{x \getsr \dist}{\query(x)}.
$$
The goal is to design an \emph{oracle} $\oracle$ that answers statistical queries about the unknown distribution $\dist$, given only iid samples $x_{1},\dots,x_{\sample}$ from $\dist$.  In this work, we are interested in the case where the queries may be adaptively and adversarially chosen.  

Specifically, $\oracle$ is a stateful algorithm that holds a tuple of samples $x_1,x_2,\dots \in \bits^{*}$, takes a statistical query $\query$ as input, and returns a real-valued answer $a \in [0,1]$.  We require that when $x$ consists of iid samples from $\dist$, the answer $a$ is close to $\query(\dist)$, and moreover that this condition holds for every query in an adaptively chosen sequence $\query_{1},\query_{2},\dots$.  Formally, we define the accuracy guarantee using the following game with a stateful adversary $\accadv$.
\begin{figure}[ht]
\begin{framed}
\begin{algorithmic}
\STATE{$\accadv$ chooses a distribution $\dist$ over $\bits^{\dimension}$}
\STATE{Sample $x_1,\dots,x_{\sample} \getsr \dist$, let $x = (x_1,\dots,x_{\sample})$}
\STATE{For $j = 1,\dots,\queries$}
\INDSTATE[1]{$\accadv(\query_{1}, a_{1}, \dots, \query_{j-1}, a_{j-1})$ outputs a query $\query_j$}
\INDSTATE[1]{$\oracle(x, \query_{j})$ outputs $a_{j}$}
\end{algorithmic}
\end{framed}
\vspace{-6mm}
\caption{$\accuracygame_{\sample, \dimension, \queries}[\oracle, \accadv]$}
\end{figure}

\begin{definition}[Accurate Oracle] \label{def:accurateoracle}
An oracle $\oracle$ is \emph{$(\alpha,\beta)$-accurate for $\queries$ adaptively chosen queries given $\sample$ samples in $\bits^{\dimension}$} if for every adversary $\accadv$,
$$
\Prob{\accuracygame_{\sample, \dimension, \queries}[\oracle, \accadv]}{\forall
j \in [\queries] \; \left| \oracle(x, \query_{j}) - \query_j(\dist) \right| \leq \alpha}
\geq 1 - \beta\,.
$$
As a shorthand, we will say that $\oracle$ is \emph{$\alpha$-accurate for $\queries$ queries} if for every $\sample, \dimension \in \N$, $\oracle$ is $(\alpha,o_{\sample}(1))$-accurate for $\queries$ queries given $\sample$ samples in $\bits^{\dimension}$.  Here, $\queries$ may depend on $\sample$ and $\dimension$ and $o_{\sample}(1)$ is a function of $\sample$ that tends to $0$.
\end{definition}

We are interested in oracles that are both accurate and computationally efficient.  We say that an oracle $\oracle$ is \emph{computationally efficient} if when given samples $x_1,\dots,x_{\sample} \in \bits^{\dimension}$ and a query $\query \from \bits^{\dimension} \to \bits$ it runs in time $\poly(\sample, \dimension, |\query|)$.  Here $\query$ will be represented as a circuit that evaluates $\query(x)$ and $|\query|$ denotes the size of this circuit.

\subsection{Fingerprinting codes}
Collusion-resilient fingerprinting codes were introduced by Boneh and Shaw~\cite{BonehS98} for the problem of watermarking digital content.  A \emph{fingerprinting code} is a pair of efficient algorithms $(\fpcgen, \fpctrace)$.  The code generator $\fpcgen$ takes a number of users $\pop$ as input and outputs a matrix $F \in \bits^{\pop \times \length(\pop)}$, for some function $\length \from \N \to \N$.  We think of $F$ as consisting of $\pop$ \emph{codewords}, one for each \emph{user} $i \in [\pop]$, with each codeword being of length $\length = \length(\pop)$.  For a subset of users $S \subseteq [\pop]$, we use $F_{S}$ to denote the $|S| \times \length$ matrix consisting of the subset of codewords belonging to users in $S$.

The security property says that any codeword can be ``traced'' to its corresponding user.  Moreover, the code is \emph{fully collusion-resilient}---if any subset of users $S \subseteq [\pop]$ ``combines'' their codewords in an arbitrary manner, then the combined codeword $a \in \bits^{\length}$ can also be traced to one of the users in $S$, provided that the combined codeword is ``consistent'' with $F_{S}$ in a very weak sense.  For the standard definition of fingerprinting codes, the consistency condition would require that for every column $j$ of $F_{S}$, if every entry of the $j$-th column shares the same bit $b$, then the $j$-th entry of $a$ is also $b$.  Formally, we will use the condition that for every $j$, $| a_j  - \Ex{i \in S}{F_{S}(i,j)} | \leq 1/3$.  For our results we require a stronger, \emph{error-robust fingerprinting code}, that can trace combined codewords that only respect a relaxed consistency condition, in which the above constraint on $a$ is only required to hold for 99\% of columns $j$.

Specifically, for any set of codewords $F_{S}$, we define
\begin{equation*}
\consistent(F_{S}) = \set{a \in \bits^{\length} \left| \, \textrm{for $.99\length$ choices of $j$, $\left|a_j - \Ex{i \in S}{F_{S}(i,j)}\right| \leq 1/3$} \right.}
\end{equation*}

We can now formally define error-robust fingerprinting codes
\begin{definition}
For a function $\length \from \N \to \N$, a pair of efficient algorithms $(\fpcgen, \fpctrace)$ is an error-robust fingerprinting code of length $\length$ if 
\begin{enumerate}
\item for every $\pop \in \N$, $\fpcgen(1^{\pop})$ outputs a matrix $F \in \bits^{\pop \times \length(\pop)}$ and
\item for every (possibly randomized) adversary $\fpcadv$ and every $S \subseteq [\pop]$, if $a \getsr \fpcadv(F_{S})$, then
$$
\Prob{F \getsr \fpcgen(1^{\pop})}{a \in \consistent(F_{S}) \land \fpctrace(F, a) \not\in S} \leq \negl(\pop).
$$
\end{enumerate}
\end{definition}

Bun, Ullman, and Vadhan~\cite{BunUV14} introduced error-robust fingerprinting codes.  They gave a construction with nearly-optimal length, building on the nearly-optimal construction of standard (non-robust) fingerprinting codes by Tardos~\cite{Tardos03}.

\begin{theorem}[\cite{BunUV14}, building on~\cite{Tardos03}] \label{thm:robustfpcsexist}
For every $\pop \in \N$, there exists an error-robust fingerprinting code of length $\length(\pop) = \tilde{\Theta}(\pop^2)$.
\end{theorem}

For our results, we will need an additional technical lemma about the fingerprinting code in~\cite{BunUV14} that we will use for our results.  The lemma states that if $|S|$ is at least a sufficiently large constant, then for most columns $j$, the mean of the $j$-th column of $F$ and that of $F_{S}$ are close.  In order to prove the lemma, we need to partially describe the algorithm $\fpcgen$.
\begin{figure}[ht]
\begin{framed}
\begin{algorithmic}
\STATE{Choose parameters $p_{1},\dots,p_{\length} \in [0,1]$.  (The parameters $p_j$ must be chosen from a particular distribution, which is not relevant for our purposes)}
\STATE{For every $i \in [\pop], j \in [\length]$, let $F(i,j) = 1$ with probability $p_j$, independently}
\end{algorithmic}
\end{framed}
\vspace{-6mm}
\caption{$\fpcgen$}
\end{figure}

\begin{lemma} \label{lem:fpcpoptosample}
For every $\pop \geq 500$, and every $S \subseteq [\pop]$ such that $|S| \geq
500$, we have
$$
\Prob{F \getsr \fpcgen}{\textrm{for $.99\length$ choices of $j$, $\left|\Ex{i \in [\pop]}{F(i,j)} - \Ex{i \in S}{F_{S}(i,j)}\right| \leq 1/6$}}\geq 1-\negl(\pop)
$$
\end{lemma}
\begin{proof}
Fix any $p_{1},\dots,p_{\length} \in [0,1]$ and any $j \in [\length]$.  Then
for every $i \in [S]$, $F(i,j)$ is an independent Bernoulli random variable
with success probability $p_j$. 
Thus, by a Hoeffding bound\footnote{For independent random variables
$X_1,\dots,X_{m}\in[0,1]$, if $X = \frac{1}{m}\sum_{i=1}^{m} X_i$, then
$\prob{X > \ex{X}+ \tau} \leq \exp(-2\tau^2m)$ and
$\prob{|X-\ex{X}|>\tau}\le2\exp(-2\tau^2m)$.} 
\begin{align*}
\prob{\left|\Ex{i \in [\pop]}{F(i,j)} - p_j\right| > 1/12} &\leq
2\exp\left(-\pop/72\right) \leq 1/400\\
\text{and}\qquad\prob{\left|\Ex{i \in S}{F(i,j)} - p_j\right| > 1/12} &\leq
2\exp\left(-|S|/72\right) \leq 1/400 .
\end{align*}
Thus, by a triangle inequality, it holds that
$$
\prob{\left|\Ex{i \in S}{F(i,j)} - \Ex{i \in [\pop]}{F(i,j)} \right| > 1/6} \leq 1/200
$$
If we let $B_j$ be the indicator variable of the event $\big\{ |\Ex{i \in S}{F(i,j)} - \Ex{i \in
[\pop]}{F(i,j)}| > 1/6\big\}$, then $\ex{\sum_{i=1}^{\length} B_j} \leq
\length/200$.  Since the parameters $p_j$ are fixed, the events $B_j$ are
independent.  Thus, by a Hoeffding bound,
$\prob{\sum_{i=1}^{\length} B_j > \length/100} \leq
\exp\left(-\length/20000\right) \leq \negl(\pop).$
The final inequality holds because $\length = \length(\pop) = \tilde{\Theta}(\pop^2)$.  Since the conclusion holds for every fixed choice of parameters $p_j$, it also holds when the parameters $p_j$ are chosen randomly as in $\fpcgen$.
\end{proof}

\section{Lower bound for natural oracles} \label{sec:natural}
In this section we prove our main result in the special case where the oracle
satisfies a natural condition, roughly speaking, that it does not evaluate a
given query outside its sample. The proof is technically simpler in this case
as it is unconditional and does not rely on any cryptographic constructions.
Nevertheless, the proof outline is essentially the same as in the general
case and so it is instructive to begin with this special case.

\begin{definition}
An oracle ${\cal O}$ is \emph{natural} if for every input sample $S$ and 
every two queries $q$ and $q'$ such that $q(x)=q'(x)$ for all $x\in S,$
the answers $a$ and $a'$ that the oracle gives on queries $q$ and $q',$
respectively, are identical if the oracle is deterministic and identically
distributed if the oracle is randomized.  If the oracle is stateful, then this condition should hold when the oracle is in any of its possible states.
\end{definition}

We will now show that there is no \emph{natural} oracle that is
accurate for a sufficiently large number of adaptively chosen queries.  To do
so, we will construct an adversary that chooses a distribution $\dist$, and
then issues queries to the oracle in such a way that any computationally
efficient oracle that is given samples from $\dist$ will fail to answer all
queries correctly.  

The adversary is described in Figure~\ref{fig:attack} and proceeds in three phases.  
In the first phase the adversary chooses the distribution $\dist$ randomly.
Then the oracle is given samples from $\dist$ and the adversary performs a
recovery phase in order to identify (most of) the samples the oracle received.
Finally, the adversary uses knowledge of (most of) the samples to find a query
that the oracle cannot answer accurately.  
In the figure, $(\fpcgen, \fpctrace)$ is a fingerprinting code of
length $\length(\pop)$.  

\begin{figure}[ht]
\begin{framed}
\begin{algorithmic}
\STATE{Given a parameter $\sample$, let $\pop = 2000\sample$ and $R=n-500.$} 
\STATE{Let $\dist$ be the uniform distribution over $\{1,\dots,p\}.$}
\STATE{Choose samples $x_1,\dots,x_{\sample} \getsr \dist$, let $x =
(x_{1},\dots,x_{\sample})$.}
\STATE{Let $S \subseteq [\pop]$ be the set of unique numbers appearing in~$x.$}
\STATE{}
\STATE{Recovery phase:}
\INDSTATE[1]{Sample fingerprinting codes $F^{1},\dots,F^{R} \getsr \fpcgen(1^{\pop})$
of length $\length = \length(\pop)$.}
\INDSTATE[1]{Let $T^0 = \emptyset$}
\INDSTATE[1]{For round $r = 1$ to $R$:}
\INDSTATE[2]{For $j = 1,\dots, \length \colon$}
\INDSTATE[3]{Define the query $\query^{r}_{j}(i)$ to be $F^{r}(i, j)$ if
$i\not\in T^{r-1}$ and $0$ otherwise.}
\INDSTATE[3]{Let $a^{r}_{j} = \oracle(x; \query^{r}_{j})$}
\INDSTATE[2]{Let $a^{r} = (a^{r}_{1},\dots, a^{r}_{\length})$}
\INDSTATE[2]{Let $i^{r} = \fpctrace(F^{r}, a^{r})$, and let $T^{r} = T^{r-1} \cup \{i^r\}$}
\STATE{}
\STATE{Attack phase:}
\INDSTATE[1]{Let $\phi = 0$ with probability $1/2$ and $\phi = 1/500$ with probability $1/2$}
\INDSTATE[1]{Sample a random subset $B\subseteq[\pop]$ of size $\phi \cdot \pop.$}
\INDSTATE[1]{Let $m_i = 1$ for all $i\in B$ and $0$ for all $i\in
[\pop]\setminus B$.}
\INDSTATE[1]{Define the query $\query^{*}(i)$ to be $m_i$ if $i \not\in
T^{R}$ and $0$ otherwise.}
\INDSTATE[1]{Let $a^* = \oracle(x, \query^*)$.}
\end{algorithmic}
\end{framed}
\vspace{-6mm}
\caption{$\realgame_{\sample}[\oracle]$}
\label{fig:attack-nat}
\end{figure}

\subsection{Analysis of the recovery phase} \label{sec:recoveryanalysis-nat}
The goal of the recovery phase of the algorithm is to identify most of the
samples $x_{1},\dots,x_{\sample}$ that are held by the oracle.  Once the
attacker has this information, he can use it to find queries that distinguish
the oracle's samples from the population and force the oracle to be inaccurate.

In order to recover samples, the attacker will force the oracle to give answers that are consistent with the fingerprinting codes $F^{1},\dots,F^{R}$, which are then given to $\fpctrace$ to recover an element of the sample.  Our first claim establishes that an accurate oracle will indeed force the oracle to give answers consistent with the fingerprinting codes.
\begin{claim}\label{clm:realaccurate-nat}
If $\oracle$ is $(1/12)$-accurate for $\sample \cdot \length(2000\sample) + 1$
adaptively chosen queries, then 
$$
\Prob{\realgame_{\sample}[\oracle]}{\textrm{$\forall r \in [R]$,
for $.99 \length$ choices of $j \in [\length]$,} \atop  \left| \oracle(x,
\query^{r}_j) - \Ex{i \in S \setminus T^{r-1}}{F^{r}(i,j)} \right| \leq 1/3}
\geq 1 - o(1)\,.
$$
\end{claim}
\begin{proof}
First we show that, 
$$
\Prob{\realgame}{\forall r \in [R], j \in [\length] \; \left| \oracle(x,
\query^{r}_j) - \Ex{i \in [\pop]}{F^{r}(i,j)} \right| \leq 1/6} \geq 1-o(1)\,.
$$
Observe that by definition, for every $r, j$,
\begin{align*}
\Ex{(i, \sk_{i}) \getsr \dist}{\query^{r}_{j}(i)} 
={} &\frac{1}{\pop} \sum_{i \in [\pop] \setminus T^{r-1}} F^{r}(i,j)\,.
\end{align*}
Since $|T^{r-1}| \leq \sample$ and $F^{r}(i,j) \in \bits$, for every $r,j,$
\begin{equation} \label{eq:realconsistency0-nat}
 \left|\Ex{i \getsr \dist}{\query^{r}_{j}(i)} -
\Ex{i \in [\pop]}{F^{r}(i,j)}\right| \leq \frac{\sample}{\pop} \leq
\frac{1}{12}\,.
\end{equation}
The oracle's input $x$ consists of $\sample$ samples from
$\dist$.  Moreover, the total number of queries issued to the oracle is at
most $\queries = \sample \cdot \length(2000\sample) + 1$.  Since the oracle
is assumed to be $(1/12)$-accurate for $\queries$ queries given $\sample$ samples
in $\bits^{\dimension}$,
\begin{equation} \label{eq:realconsistency1-nat} 
\Prob{\realgame}{\forall r, j \;
\left| \oracle(x, \query^{r}_j) - \Ex{(i, \sk_{i}) \getsr \dist}{\query^{r}_{j}(i,
\sk_{i})} \right| \leq \frac1{12}} \geq 1-o(1)\,.
\end{equation}
Applying the triangle inequality to \eqref{eq:realconsistency0-nat} and
\eqref{eq:realconsistency1-nat}, this shows
\begin{equation} \label{eq:realconsistency2-nat}
\Prob{\realgame}{\forall r, j \; \left| \oracle(x, \query^{r}_j) - \Ex{i \in
[\pop]}{F^{r}(i,j)} \right| \leq 1/6} \geq 1-o(1) \,.
\end{equation}

By Lemma~\ref{lem:fpcpoptosample}, since $|S \setminus T^{r-1}| \geq 500$, for every $r$, 
$$
\Prob{}{\textrm{for $.99\length$ choices of $j$, $\left|\Ex{i \in
[\pop]}{F^{r}(i,j)} - \Ex{i \in S}{F^{r}_{S}(i,j)}\right| \leq 1/6$}}\geq
1-\negl(\sample)\,.
$$
where the probability is taken over the choice of $F^{r} \getsr \fpcgen$.
By a union bound over $r = 1,\dots,R$, where $R = \sample - 500$, if $F^{1},\dots,F^{R} \getsr \fpcgen$, then
\begin{equation} \label{eq:realconsistency3-nat}
\Prob{}{\textrm{$\forall r$, for $.99\length$ choices of $j$, $\left|\Ex{i \in
[\pop]}{F^{r}(i,j)} - \Ex{i \in S}{F^{r}_{S}(i,j)}\right| \leq 1/6$}}\geq
1-\negl(\sample)\,.
\end{equation}

The claim now follows by combining \eqref{eq:realconsistency2-nat} and
\eqref{eq:realconsistency3-nat}.
\end{proof}

Now that we have established Claim~\ref{clm:realaccurate}, we know that in
every round $r$, the oracle~$\oracle$ holding $x$ returns a set of answers
that are consistent with the fingerprinting code $F^{r}_{S \setminus
T^{r-1}}$. However, this fact alone is not enough to guarantee that
$\fpctrace$ returns a user in $S \setminus T^{r-1}$, because the queries to
the oracle depend on rows of $F^{r}$ for users outside of $S \setminus
T^{r-1}$, whereas the security of the fingerprinting code applies only to
algorithms that only have access to the rows of $F^{r}$ for users in $S
\setminus T^{r-1}$. However, if we assume that the oracle is natural, then its answers do not depend on information about the query at points outside of the sample~$S.$

\begin{lemma} \label{lem:realrecoverysucceeds-nat}
If $\oracle$ is a natural oracle and is $(1/12)$-accurate for $\sample
\cdot \length(2000\sample) + 1$ adaptively chosen queries, then 
$$
\Prob{\realgame_{\sample,\dimension}[\oracle]}{|S \setminus T^{R}| > 500} = o(1)
$$
\end{lemma}

\begin{proof}
Fix any round $r \in \set{1,\dots,|S|-500}$ and let $U = S \setminus T^{r-1}$.
By the security of the fingerprinting code, we have that for every algorithm $\cA$
\begin{equation*}
\Prob{F^r \getsr \fpcgen(1^{\pop})}{\textrm{$($$\cA(F^{r}_{U}) \in
\consistent(F^{r}_{U})$$)$ $\land$ $($$\fpctrace(F^{r}, \cA(F^{r}_{U}))
\not\in U$$)$}}  \le \negl(\sample)
\end{equation*}
Observe that the oracle $\oracle$ is natural and therefore the answer it gives
on any query cannot depend on rows of $F^{r}$ that belong to users outside
of~$S.$ Moreover, the query is $0$ on points in $T^{r-1}$ and  
The queries issued in rounds $r' \neq r$ depend only on $F^{r'}$, which is
independent from $F^{r}$. Hence, the answer of
the oracle depends only on points in~$U.$ We therefore have
\begin{equation}
\Prob{F^r \getsr \fpcgen(1^{\pop})}{\textrm{$($$a^{r} \in
\consistent(F^{r}_{U})$$)$ $\land$ $($$\fpctrace(F^{r}, a^{r}) \not\in U$$)$}}
\le \negl(\sample)\,.
\end{equation}
By a union bound over $r = 1,\dots,|S|-500$, we also have
\begin{equation} \label{eq:trace0}
\Prob{F^{1},\dots,F^{|S|-500} \getsr \fpcgen(1^{\pop})}{
\begin{array}{ll}
\exists r &   \left(a^{r} \in \consistent(F^{r}_{S \setminus
T^{r-1}})\right)\\
& \land \left(\fpctrace(F^{r}, a^{r}) \not\in S \setminus T^{r-1}\right)
\end{array}} 
\le \negl(\sample)
\end{equation}

By Claim~\ref{clm:realaccurate-nat}, we have that with probability at least $1 -
o(1)$, it holds that for all $r \in \set{1,\dots,|S|-500}$ and for $.99 \length$ choices of $j \in [\length]$
\begin{equation*}
\left| a^{r}_{j} - \Ex{i \in U}{F^{r}(i,j)} \right| \leq 1/3
\end{equation*}
Note that in order to apply Claim~\ref{clm:realaccurate-nat} we have used the fact that when 
$r \in \set{1,\dots,|S|-500}$, $|S \setminus T^{r-1}| \geq 500$.
If this condition is satisfied, then indeed $a^{r} \in \consistent(F^{r}_{U})$.  
Therefore, combining with \eqref{eq:trace0}, we have
\begin{equation*}
\Prob{F^{1},\dots,F^{|S|-500} \getsr \fpcgen(1^{\pop})}{\exists r,
\fpctrace(F^{r}, a^{r}) \not\in S \setminus T^{r-1}} \leq o(1)\,.
\end{equation*}
Now the claim follows by observing that 
\begin{align*}
\Prob{\realgame}{|S \setminus T^{R}| > 500} 
\leq\Prob{\realgame}{\exists r \in \set{1,\dots,|S|-500} \; \fpctrace(F^{r},
a^{r}) \not\in S \setminus T^{r-1}} \leq o(1)\,.
\end{align*}
\end{proof}

\subsection{Analysis of the Attack Phase} \label{sec:attackanalysis-nat}
At this point we know that if the oracle is natural and 
accurately answers all the queries in the recovery phase, then with high probability $|S \setminus T^{R}| \leq 
500$. Next we show that if this event occurs, then the probability that the oracle 
answers $\query^*$ accurately in the attack phase is bounded away from $1$ by a constant.  Since an accurate oracle
is required to answer each query accurately with probability at least $1-o(1)$, we will obtain a contradiction.

To begin with we show that the population answer~$\query^*(\dist)$ is close to the
value~$\phi$ in the attack.
\begin{claim}\label{clm:attackp-nat}
In $\realgame_{\sample}[\oracle],$ we have $|\query^*(\dist)-\phi|\le 1/2000.$ 
\end{claim}
\begin{proof}
The case $\phi=0$ we have $\query^*(\dist)=0.$ If $\phi=1/500,$ then we have that $\E_{i\in[\pop]} m_i =
\phi$ since $|B|=\phi\cdot \pop.$ Hence,
\[
\query^*(\dist)
= \E_{i\sim\dist} \query^*(i)
= \phi - \Prob{i \getsr [\pop]}{i\in T^R \land m_i = 1}
\]
where $\prob{i\in T^R \land m_i = 1} \leq \prob{i \in T^{R}}= |T^R|/\pop=(\sample-500)/\pop \le 1/2000.$
\end{proof}

We will now show that the oracle cannot guess the value of $\phi$ with sufficiently high 
probability provided that the recovery phase succeeded.
\begin{claim} \label{clm:realattacksucceeds-nat}
$$
\Prob{\realgame_{\sample}[\oracle]}{\left( |S \setminus T^{R}| \leq 500 \right) \land \left|\oracle(x, \query^{*}) -
\phi\right| \leq \frac{1}{2000}} \leq 1 - \Omega(1)
$$
\end{claim}
\begin{proof}
Consider the case where $\phi=1/500.$  We have
\begin{align}
&\prob{\forall i \in S \setminus T^{R}, \; m_i = 0 \left|\, |S \setminus T^{R}| \leq 500 \right.}
\ge \prod_{i=0}^{499} \left( \frac{(1-\phi)\pop - i}{\pop - i} \right) \notag \\
\ge &\left(\frac{(1-\phi)\pop - 499}{\pop-499} \right)^{500}
= \left(1 - \phi \frac{\pop}{\pop - 499} \right)^{500}
\ge \left(1 - 2\phi\right)^{500} 
= \left(1 - \frac{1}{250} \right)^{500} \geq \frac{1}{4e^2} \notag \notag
\end{align}
where we used that $\pop \geq 2000.$
On the other hand, when $\phi = 0$, we have
\[
\prob{\forall i \in S \setminus T^{R}, \; m_i = 0 \left|\, |S \setminus T^{R}| \leq 500
\right.} = 1.
\]
Note that because the oracle is natural it answer only depends on $m_i$ for $i\in S\setminus T^R.$ 
When the oracle sees only that $m_i = 0$ for every $i \in S \setminus T^{R}$,
it cannot give an answer that is simultaneously accurate to within $1/2000$
for both the case of $\phi = 0$ and for the case of $\phi = 1/500$. The
event $m_i = 0$ for every $i \in S \setminus T^{R}$ occurs with at least
probability $1/4e^2$ as shown above. Conditioned on this event, both cases $\phi=0$
and $\phi=1/500$ have constant probability. Hence, the answer of the oracle must be
far from $\phi$ with constant probability.  Formally,
\[
\Prob{\realgame}{\left.\left|\oracle(x, \query^{*}) - \phi\right| \leq
\frac{1}{2000}\,\right|\, |S \setminus T^{R}| \leq 500} \leq 1 - \Omega(1)\,.
\]
By Lemma\label{lem:realrecoverysucceeds-nat}, this implies
\[
\Prob{\realgame}{\left(\left|\oracle(x, \query^{*}) - \phi\right| \leq
\frac{1}{2000}\right) \land (|S \setminus T^{R}| \leq 500)} \leq 
\frac{1 - \Omega(1)}{1-o(1)}\leq 1-\Omega(1)\,.
\]
\end{proof}

\begin{lemma} \label{lem:realattacksucceeds2-nat}
If $\oracle$ is natural, then 
$$
\Prob{\realgame_{\sample,\dimension}[\oracle]}{(|S \setminus T^{R}| \leq 500) \land \left|\oracle(x, \query^{*}) -
\query^{*}(\dist)\right| \leq \frac{1}{2000}} \leq 1 - \Omega(1)
$$
\end{lemma}
\begin{proof}
By Claim~\ref{clm:attackp-nat}, we have that $|\query^*(\dist)-\phi|\le 1/2000.$
Claim~\ref{clm:realattacksucceeds-nat} shows that
\[
\Prob{\realgame}{(|S \setminus T^{R}| \leq 500) \land \left|\oracle(x, \query^{*}) -
\phi\right| \leq \frac{1}{2000}} \leq 1 - \Omega(1)\,.
\]
The statement of the lemma now follows from a triangle inequality.
\end{proof}

\subsection{Putting it together}

\begin{theorem}
There is no natural oracle $\oracle$ that is $(1/2000)$-accurate for $\sample^{3+o(1)}$ adaptively chosen queries given $\sample$ samples.
\end{theorem}

\begin{proof}
The entire attack will consist of $\queries = R \cdot \length + 1$ queries, where $R = \sample - 500$ and $\length = \tilde{O}(\pop^2) = \tilde{O}(\sample^2)$.  Therefore the entire attack consists of $\queries = \sample^{3+o(1)}$ queries.  

Therefore, if $\oracle$ is natural, by Lemma~\ref{lem:realrecoverysucceeds-nat},
$$
\Prob{\realgame_{\sample}[\oracle]}{|S \setminus T^{R}| > 500} = o_{\sample}(1)
$$
By Lemma~\ref{lem:realattacksucceeds2-nat},
$$
\Prob{\realgame_{\sample}[\oracle]}{(|S \setminus T^{R}| \leq 500) \land \left|\oracle(x, \query^{*}) -
\query^{*}(\dist)\right| \leq \frac{1}{2000}} \leq 1 - \Omega(1)
$$
Combining these two statements gives
$$
\Prob{\realgame_{\sample}[\oracle]}{ \left|\oracle(x, \query^{*}) - \query^{*}(\dist)\right| \leq
\frac{1}{2000}} \leq 1 - \Omega(1)
$$
However, the definition of an accurate oracle asserts, in particular, that
$$
\Prob{\realgame_{\sample}[\oracle]}{ \left|\oracle(x, \query^{*}) - \query^{*}(\dist)\right| \leq \frac{1}{2000}}
\geq 1-o(1)
$$
and thereby we obtain a contradiction.
\end{proof}

\section{Lower bound for all computationally bounded oracles} \label{sec:attack}
In this section we will show that there is no computationally efficient oracle that is accurate for a sufficiently large number of adaptively chosen queries, and thereby formally establish Theorem~\ref{thm:main1} in the introduction.  
To do so, we will construct an adversary that chooses a distribution $\dist$, and then issues queries to the oracle such that no computationally efficient oracle given samples from $\dist$ can answer all the queries correctly.  

\subsection{Encryption schemes}
Our attack relies on the existence of a semantically secure private-key
encryption scheme that we briefly recall here.  
An encryption scheme is a triple of efficient algorithms $(\encgen, \encenc, \encdec)$ with the following syntax:
\begin{itemize}
\item $\encgen$ is a randomized algorithm that a security parameter $\security$ and outputs an $\keylength(\security)$-bit secret key for some non-decreasing function $\keylength \from \N \to \N$.  Formally, $\sk \getsr \encgen(1^\security)$.
\item $\encenc$ is a randomized algorithm that takes as input a secret key and a one-bit message $m \in \bits$ and outputs a ciphertext $\ct$.  Formally, $\ct \getsr \encenc(\sk, m)$.
\item $\encdec$ is a deterministic algorithm that takes as input a secret key and a ciphertext $\ct$ and outputs a decryption $m'$.  If the ciphertext $\ct$ was an encryption of $m$ under the key $\sk$, then $m' = m$.  Formally, if $\ct \getsr \encenc(\sk, m)$, then $\encdec(\sk, \ct) = m$ with probability $1$.
\end{itemize}

Roughly, security of the encryption scheme asserts that no polynomial time
adversary who does not know the secret key can distinguish encryptions of $m =
0$ from encryptions of $m = 1$, even if the adversary has access to an oracle that returns the
encryption of an arbitrary message under the unknown key.  For
convenience, we will require that this security property holds simultaneously
for an arbitrary polynomial number of secret keys. The existence of an encryption scheme with
this property follows immediately from the existence an ordinary semantically secure encryption scheme.  We 
start with the stronger definition only to simplify our proofs. A secure encryption scheme exists
under the minimal cryptographic assumption that one-way functions exist. The
formal definition of security is not needed until
Section~\ref{sec:securityproofs}.

\subsection{Description of the attack}

The adversary is specified in Figure~\ref{fig:attack}.  The adversary works in three phases.  In the first phase the adversary chooses the distribution $\dist$ randomly.  Then the oracle is given samples from $\dist$ and the adversary performs a recovery phase in order to identify (most of) the samples the oracle received.  Finally, the adversary uses knowledge of (most of) the samples to find a query that the oracle cannot answer accurately.  See Section~\ref{sec:proofoverview} for more informal description of the adversary.  In Figure~\ref{fig:attack}, $(\encgen, \encenc, \encdec)$ is an encryption scheme with key length $\keylength(\security)$ and $(\fpcgen, \fpctrace)$ is a fingerprinting code of length $\length(\pop)$.  Observe that $\realgame_{\sample, \dimension}$ is only well defined for pairs $n, d \in N$ such that there exists $\security \in \N$ for which $\keylength(1) + \lceil \log(2000\sample) \rceil \leq d$.  Through this section we will assume that $\sample = \sample(\dimension) = \poly(\dimension)$ and that $\dimension$ is a sufficiently large constant, which ensures that $\realgame_{\sample, \dimension}$ is well defined.

\begin{figure}[ht]
\begin{framed}
\begin{algorithmic}
\STATE{The distribution $\dist$:}
\INDSTATE[1]{Given parameters $\dimension, \sample$, let $\pop = 2000\sample$, let $R = \sample-500$}
\INDSTATE[1]{Let $\security$ be the largest integer such that $\keylength(\lambda) + \lceil \log \pop \rceil \leq \dimension$}
\INDSTATE[1]{For $i \in [\pop]$, let $\sk_{i} \getsr \encgen(1^{\security})$ and let $y_i = (i, \sk_{i})$}
\INDSTATE[1]{Let $\dist$ be the uniform distribution over $\set{y_1,\dots,y_{\pop}} \subseteq \bits^{\dimension}$}
\STATE{}
\STATE{Choose samples $x_1,\dots,x_{\sample} \getsr \dist$, let $x = (x_{1},\dots,x_{\sample})$}
\STATE{Let $S \subseteq [\pop]$ be the set of unique indices $i$ such that $(i, \sk_{i})$ appears in $x$}
\STATE{}
\STATE{Recovery phase:}
\INDSTATE[1]{Setup fingerprinting codes and ciphertexts:}
\INDSTATE[2]{Let $F^{1},\dots,F^{R} \getsr \fpcgen(1^{\pop})$, let $\length = \length(\pop)$ be code length}
\INDSTATE[2]{For $r = 1,\dots, R$, $i = 1,\dots,\pop$, $j = 1,\dots,\length$, let $\ct^{r}(i,j) = \encenc(\sk_i, F^{r}(i,j))$}
\INDSTATE[1]{Let $T^0 = \emptyset$}
\INDSTATE[1]{For round $r = 1,\dots,R$:}
\INDSTATE[2]{For $j = 1,\dots, \length \colon$}
\INDSTATE[3]{Define the query $\query^{r}_{j}(i', \sk')$ to be $\encdec(\sk',\ct^{r}(i', j))$ if $i' \not\in T^{r-1}$ and $0$ otherwise}
\INDSTATE[3]{Let $a^{r}_{j} = \oracle(x; \query^{r}_{j})$}
\INDSTATE[2]{Let $a^{r} = (a^{r}_{1},\dots, a^{r}_{\length})$}
\INDSTATE[2]{Let $i^{r} = \fpctrace(F^{r}, a^{r})$, and let $T^{r} = T^{r-1} \cup \{i^r\}$}
\STATE{}
\STATE{Attack phase:}
\INDSTATE[1]{Let $\phi = 0$ with probability $1/2$ and $\phi = 1/500$ with probability $1/2$}
\INDSTATE[1]{Sample a random subset $B\subseteq[\pop]$ of size $\phi \cdot \pop.$}
\INDSTATE[1]{Let $m_i = 1$ for all $i\in B$ and $0$ for all $i\in [\pop]\setminus B$}
\INDSTATE[1]{Let $\ct^*_{i} = \encenc(\sk_i, m_i)$ for all $i\in[\pop]$}
\INDSTATE[1]{Define the query $\query^{*}(i', \sk')$ to be $\encdec(\sk',\ct^{*}_{i'})$ if $i' \not\in T^{R}$ and $0$ otherwise}
\INDSTATE[1]{Let $a^* = \oracle(x, \query^*)$}
\end{algorithmic}
\end{framed}
\vspace{-6mm}
\caption{$\realgame_{\sample, \dimension}[\oracle]$}
\label{fig:attack}
\end{figure}

\subsection{Analysis of the recovery phase} \label{sec:recoveryanalysis}
The goal of the recovery phase of the algorithm is to identify most of the samples $x_{1},\dots,x_{\sample}$ that are held by the oracle.  Once the attacker has this information, he can use it to find queries that distinguish the oracle's keys from the population and force the oracle to be inaccurate.

In order to recover keys, the attacker will force the oracle to give answers that are consistent with the fingerprinting codes $F^{1},\dots,F^{R}$, which are then given to $\fpctrace$ to recover an element of the sample.  Our first claim establishes that an accurate oracle will indeed force the oracle to give answers consistent with the fingerprinting codes.
\begin{claim}\label{clm:realaccurate}
If $\oracle$ is $(1/12)$-accurate for $\sample \cdot \length(2000\sample) + 1$ adaptively chosen queries then for every polynomial $\sample = \sample(\dimension)$ and every sufficiently large $\dimension \in \N$,
$$
\Prob{\realgame_{\sample, \dimension}[\oracle]}{\textrm{$\forall r \in [R]$, for $.99 \length$ choices of $j \in [\length]$,} \atop  \left| \oracle(x, \query^{r}_j) - \Ex{i \in S \setminus T^{r-1}}{F^{r}(i,j)} \right| \leq 1/3} \geq 1 - o(1)
$$
\end{claim}
\begin{proof}
First we show that, 
$$
\Prob{\realgame}{\forall r \in [R], j \in [\length] \; \left| \oracle(x, \query^{r}_j) - \Ex{i \in [\pop]}{F^{r}(i,j)} \right| \leq 1/6} \geq 1-o(1)
$$
Observe that by definition, for every $r, j$,
\begin{align*}
&\Ex{(i, \sk_{i}) \getsr \dist}{\query^{r}_{j}(i, \sk_{i})} \\
={} &\frac{1}{\pop} \sum_{i \in [\pop] \setminus T^{r-1}} \encdec(\sk_{i}, \ct^{r}(i,j)) + \frac{1}{\pop} \sum_{i \in T^{r-1}} 0 \\
={} &\frac{1}{\pop} \sum_{i \in [\pop] \setminus T^{r-1}} F^{r}(i,j)
\end{align*}
where the last equality is because $\ct^{r}(i,j) \getsr \encenc(\sk_i, F^{r}(i,j))$.
Since $|T^{r-1}| \leq \sample$ and $F^{r}(i,j) \in \bits$,
\begin{equation} \label{eq:realconsistency0}
\forall r, j \; \left|\Ex{(i,\sk_{i}) \getsr \dist}{\query^{r}_{j}(i, \sk_{i})} -
\Ex{i \in [\pop]}{F^{r}(i,j)}\right| \leq \frac{\sample}{\pop} \leq
\frac{1}{12}
\end{equation}
In $\realgame$, the oracle's input $x$ consists of $\sample$ samples from
$\dist$.  Moreover, the total number of queries issued to the oracle is at
most $\queries = \sample \cdot \length(2000\sample) + 1$.  Since the oracle
is assumed to be $(1/12)$-accurate for $\queries$ queries given $\sample$ samples
in $\bits^{\dimension}$,
\begin{equation} \label{eq:realconsistency1} 
\Prob{\realgame}{\forall r, j \;
\left| \oracle(x, \query^{r}_j) - \Ex{(i, \sk_{i}) \getsr \dist}{\query^{r}_{j}(i,
\sk_{i})} \right| \leq \frac1{12}} \geq 1-o(1)
\end{equation}
Applying the triangle inequality to \eqref{eq:realconsistency0} and
\eqref{eq:realconsistency1}, this shows
\begin{equation} \label{eq:realconsistency2}
\Prob{\realgame}{\forall r, j \; \left| \oracle(x, \query^{r}_j) - \Ex{i \in
[\pop]}{F^{r}(i,j)} \right| \leq 1/6} \geq 1-o(1) \,.
\end{equation}

By Lemma~\ref{lem:fpcpoptosample}, since $|S \setminus T^{r-1}| \geq 500$, for every $r$, if $F^{r} \getsr \fpcgen(1^{\pop})$, then
$$
\Prob{}{\textrm{for $.99\length$ choices of $j$, $\left|\Ex{i \in [\pop]}{F^{r}(i,j)} - \Ex{i \in S}{F^{r}_{S}(i,j)}\right| \leq 1/6$}}\geq 1-\negl(\sample)
$$
where the probability is taken over the choice of $F^{r} \getsr \fpcgen$.
By a union bound over $r = 1,\dots,R$, where $R = \sample - 500$, if $F^{1},\dots,F^{R} \getsr \fpcgen$, then
\begin{equation} \label{eq:realconsistency3}
\Prob{}{\textrm{$\forall r$, for $.99\length$ choices of $j$, $\left|\Ex{i \in [\pop]}{F^{r}(i,j)} - \Ex{i \in S}{F^{r}_{S}(i,j)}\right| \leq 1/6$}}\geq 1-\negl(\sample)
\end{equation}

The claim now follows by combining \eqref{eq:realconsistency2} and \eqref{eq:realconsistency3}.
\end{proof}

Now that we have established Claim~\ref{clm:realaccurate}, we know that in
every round $r$, the oracle~$\oracle$ holding $x$ returns a set of answers that are consistent with the fingerprinting code $F^{r}_{S \setminus T^{r-1}}$.  However, this fact alone is not enough to guarantee that $\fpctrace$ returns a user in $S \setminus T^{r-1}$, because the queries to the oracle depend on rows of $F^{r}$ for users outside of $S \setminus T^{r-1}$, whereas the security of the fingerprinting code applies only to algorithms that only have access to the rows of $F^{r}$ for users in $S \setminus T^{r-1}$.  To remedy this problem we rely on the fact that the rows of $F^{r}$ outside of $S \setminus T^{r-1}$ are encrypted under keys $\sk$ that are not known to the oracle.  Thus, a computationally efficient oracle ``does not know'' those rows.  We can formalize this argument by comparing $\realgame$ to an $\idealgame$ where rows of $F^{r}$ for users outside of $S \setminus T^{r-1}$ are replaced with zeros, and argue that the adversary cannot distinguish between these two attacks without breaking the security of the encryption scheme.

\begin{figure}[ht]
\begin{framed}
\begin{algorithmic}
\STATE{The distribution $\dist$:}
\INDSTATE[1]{Given parameters $\dimension, \sample$, let $\pop = 2000 \sample$, let $R = \sample-500$}
\INDSTATE[1]{Let $\security$ be the largest integer such that $\keylength(\lambda) + \lceil \log \pop \rceil \leq \dimension$}
\INDSTATE[1]{For $i \in [\pop]$, let $\sk_{i} \getsr \encgen(1^{\security})$ and let $y_i = (i, \sk_{i})$}
\INDSTATE[1]{Let $\dist$ be the uniform distribution over $\set{y_1,\dots,y_{\pop}} \subseteq \bits^{\dimension}$}
\STATE{}
\STATE{Choose samples $x_1,\dots,x_{\sample} \getsr \dist$, let $x = (x_{1},\dots,x_{\sample})$}
\STATE{Let $S \subseteq [\pop]$ be the set of unique indices $i$ such that $(i, \sk_{i})$ appears in $x$}
\STATE{}
\STATE{Recovery phase:}
\INDSTATE[1]{Setup fingerprinting codes and ciphertexts:}
\INDSTATE[2]{Let $F^{1},\dots,F^{R} \getsr \fpcgen(1^{\pop})$, let $\length = \length(\pop)$ be code length}
\INDSTATE[2]{For $r = 1,\dots, R$, {\color{red} $i \in S$}, $j = 1,\dots,\length$, let $\ct^{r}(i,j) = \encenc(\sk_i, F^{r}(i,j))$}
\INDSTATE[2]{{\color{red} For $r = 1,\dots, R$, $i \in [\pop] \setminus S$, $j = 1,\dots,\length$, let $\ct^{r}(i,j) = \encenc(\sk_i, 0)$}}
\INDSTATE[1]{Let $T^0 = \emptyset$}
\INDSTATE[1]{For round $r = 1,\dots,R$:}
\INDSTATE[2]{For $j = 1,\dots, \length \colon$}
\INDSTATE[3]{Define the query $\query^{r}_{j}(i', \sk')$ to be $\encdec(\sk',\ct^{r}(i', j))$ if $i' \not\in T^{r-1}$ and $0$ otherwise}
\INDSTATE[3]{Let $a^{r}_{j} = \oracle(x; \query^{r}_{j})$}
\INDSTATE[2]{Let $a^{r} = (a^{r}_{1},\dots, a^{r}_{\length})$}
\INDSTATE[2]{Let $i^{r} = \fpctrace(F^{r}, a^{r})$, and let $T^{r} = T^{r-1} \cup \{i^r\}$}
\STATE{}
\STATE{Attack phase:}
\INDSTATE[1]{Let $\phi = 0$ with probability $1/2$ and $\phi = 1/500$ with probability $1/2$}
\INDSTATE[1]{Sample a random subset $B\subseteq[\pop]$ of size $\phi\cdot \pop.$}
\INDSTATE[1]{Let $m_i = 1$ for all $i\in B$ and $0$ for all $i\in [\pop]\setminus B$}
\INDSTATE[1]{{\color{red} For each $i \in S$, let $\ct^*_{i} = \encenc(\sk_i,
m_i)$, for each $i\in [\pop]\setminus S,$ let $\ct^*_{i} = \encenc(\sk_i, 0)$}}
\INDSTATE[1]{Define the query $\query^{*}(i', \sk')$ to be $\encdec(\sk',\ct^{*}_{i'})$ if $i' \not\in T^{R}$ and $0$ otherwise}
\INDSTATE[1]{Let $a^* = \oracle(x, \query^*)$}
\end{algorithmic}
\end{framed}
\vspace{-6mm}
\caption{$\idealgame_{\sample, \dimension}[\oracle]$}
\label{fig:idealattack}
\end{figure}

\begin{claim}\label{clm:indist1}
Let $Z_{1}$ be the event
$$
\textrm{$\forall r \in [R]$, for $.99 \length$ choices of $j \in [\length]$,}  \left| \oracle(x, \query^{r}_j) - \Ex{i \in S \setminus T^{r-1}}{F^{r}(i,j)} \right| \leq 1/3
$$
Assume $(\encgen, \encenc, \encdec)$ is a computationally secure encryption scheme with key length $\keylength(\security) = \security$ and let $\sample = \sample(\dimension)$ be any polynomial.  Then if $\oracle$ is computationally efficient, for every $\dimension \in \N$
$$
\left| \Prob{\idealgame_{\sample, \dimension}[\oracle]}{Z_{1}} - \Prob{\realgame_{\sample, \dimension}[\oracle]}{Z_{1}} \right| \leq \negl(\dimension)
$$
\end{claim}
The proof follows from the security of the encryption scheme.  We defer the details to Section~\ref{sec:securityproofs}.

\begin{claim}\label{clm:idealaccurate}
If $\oracle$ is computationally efficient and $(1/12)$-accurate for $\sample
\cdot \length(2000\sample) + 1$ adaptively chosen queries, then for any polynomial $\sample = \sample(\dimension)$, and every sufficiently large $\dimension \in \N$,
$$
\Prob{\idealgame_{\sample, \dimension}[\oracle]}{\textrm{$\forall r \in [R]$, for $.99 \length$ choices of $j \in [\length]$,} \atop  \left| \oracle(x, \query^{r}_j) - \Ex{i \in S \setminus T^{r-1}}{F^{r}(i,j)} \right| \leq 1/3} \geq 1 - o(1)
$$
\end{claim}
The proof is immediate by combining Claim~\ref{clm:realaccurate} and Claim~\ref{clm:indist1}.

Next, we argue that in $\idealgame$, with high probability $\fpctrace$ only outputs users contained in the sample $S$.
\begin{claim} \label{clm:idealrecovery}
If $\oracle$ is computationally efficient and $(1/12)$-accurate for $\sample
\cdot \length(2000\sample) + 1$ adaptively chosen queries, then for any polynomial $\sample = \sample(\dimension)$, and every sufficiently large $\dimension \in \N$,
$$
\Prob{\idealgame_{\sample,\dimension}[\oracle]}{|S \setminus T^{R}| > 500} \le o(1)
$$
\end{claim}

\begin{proof}[Proof]
Fix any round $r \in \set{1,\dots,|S|-500}$ and let $U = S \setminus T^{r-1}$.
By the security of the fingerprinting code, we have that for every algorithm $\cA$
\begin{equation*}
\Prob{F^r \getsr \fpcgen(1^{\pop})}{\textrm{$($$\cA(F^{r}_{U}) \in
\consistent(F^{r}_{U})$$)$ $\land$ $($$\fpctrace(F^{r}, \cA(F^{r}_{U}))
\not\in U$$)$}}  \le \negl(\sample)
\end{equation*}
Observe that in $\idealgame$, the oracle $\oracle$ is never given any input that depends on rows of $F^{r}$ that belong to users outside of $U$:  The queries issued in rounds $r' \neq r$ depend only on $F^{r'}$, which is independent from $F^{r}$.  And in round $r$ the query only depends on ciphertexts $\ct^{r}(i,j)$ for $i \not\in T^{r-1}$, which are all independent of $F^{r}(i,j)$ whenever $i \not\in U$.  Therefore we have
\begin{equation}
\Prob{F^r \getsr \fpcgen(1^{\pop})}{\textrm{$($$a^{r} \in
\consistent(F^{r}_{U})$$)$ $\land$ $($$\fpctrace(F^{r}, a^{r}) \not\in U$$)$}}
\le \negl(\sample)
\end{equation}
By a union bound over $r = 1,\dots,|S|-500$, we also have
\begin{equation} \label{eq:idealtrace0}
\Prob{F^{1},\dots,F^{|S|-500} \getsr \fpcgen(1^{\pop})}{
\begin{array}{ll}
\exists r &   \left(a^{r} \in \consistent(F^{r}_{S \setminus
T^{r-1}})\right)\\
& \land \left(\fpctrace(F^{r}, a^{r}) \not\in S \setminus T^{r-1}\right)
\end{array}} 
\le \negl(\sample)
\end{equation}

By Claim~\ref{clm:idealaccurate}, we have that with probability at least $1 -
o(1)$, it holds that for all $r \in \set{1,\dots,|S|-500}$ and for $.99 \length$ choices of $j \in [\length]$
\begin{equation*}
\left| a^{r}_{j} - \Ex{i \in U}{F^{r}(i,j)} \right| \leq 1/3
\end{equation*}
Note that in order to apply Claim~\ref{clm:idealaccurate} we have used the fact that when $r \in \set{1,\dots,|S|-500}$, $|S \setminus T^{r-1}| \geq 500$.
If this condition is satisfied, then indeed $a^{r} \in \consistent(F^{r}_{U})$.  Therefore, combining with \eqref{eq:idealtrace0}, we have
\begin{equation*}
\Prob{F^{1},\dots,F^{|S|-500} \getsr \fpcgen(1^{\pop})}{\exists r, \fpctrace(F^{r}, a^{r}) \not\in S \setminus T^{r-1}} \leq o(1)
\end{equation*}
Now the claim follows by observing that 
\begin{align*}
&\Prob{\idealgame_{\sample, \dimension}[\oracle]}{|S \setminus T^{R}| > 500} \\
\leq{} &\Prob{\idealgame_{\sample, \dimension}[\oracle]}{\exists r \in \set{1,\dots,|S|-500} \; \fpctrace(F^{r}, a^{r}) \not\in S \setminus T^{r-1}} \leq o(1)
\end{align*}
\end{proof}

Finally, we show that if $|S \setminus T^{R}| \leq 500$ with high probability in $\idealgame$, then $|S \setminus T^{R}| \leq 500$ with high probability in $\realgame$.  Again, we do so by arguing that $\realgame$ and $\idealgame$ are computationally indistinguishable.
\begin{claim}\label{clm:indist2}
Let $Z_{2}$ be the event
$\left\{ |S \setminus T^{R}| \leq 500\right\}.$
Assume $(\encgen, \encenc, \encdec)$ is a computationally secure encryption scheme with key length $\keylength(\security) = \security$ and let $\sample = \sample(\dimension)$ be any polynomial.  Then if $\oracle$ is computationally efficient, for every $\dimension \in \N$
$$
\left| \Prob{\idealgame_{\sample, \dimension}[\oracle]}{Z_{2}} -
\Prob{\realgame_{\sample, \dimension}[\oracle]}{Z_{2}} \right| \leq
\negl(\dimension)\,.
$$
\end{claim}
The proof follows from the security of the encryption scheme.  We defer the details to Section~\ref{sec:securityproofs}.

\begin{lemma} \label{lem:realrecoverysucceeds}
If $\oracle$ is computationally efficient and $(1/12)$-accurate for $\sample
\cdot \length(2000\sample) + 1$ adaptively chosen queries, then for any polynomial $\sample = \sample(\dimension)$, and every sufficiently large $\dimension \in \N$,
$$
\Prob{\realgame_{\sample,\dimension}[\oracle]}{|S \setminus T^{R}| > 500} = o(1)
$$
\end{lemma}
The proof is immediate by combining Claim~\ref{clm:idealrecovery} and Claim~\ref{clm:indist2}.

\subsection{Analysis of the attack phase} \label{sec:attackanalysis}
By the arguments of Section \ref{sec:recoveryanalysis}, we know that if the oracle is computationally efficient and 
accurately answers all the queries in the recovery phase, then with high probability $|S \setminus T^{R}| \leq 
500$.  In this section we will show that if these events indeed occur, then the probability that the oracle answers $
\query^*$ accurately in the attack phase is bounded away from $1$ by a constant.  Since an accurate oracle
is required to answer each query accurately with probability at least $1-o(1)$, we will have obtained a contradiction.

To begin with we show that the population answer~$\query^*(\dist)$ is close to the
value~$\phi$ in the real attack.
\begin{claim}\label{clm:attackp}
For every polynomial $\sample = \sample(\dimension)$ and every sufficiently large $\dimension \in \N$, in $\realgame_{\sample, \dimension}[\oracle],$ we have $|\query^*(\dist)-\phi|\le 1/2000.$ 
\end{claim}
\begin{proof}
The case $\phi=0$ we have $\query^*(\dist)=0.$ If $\phi=1/500,$ then we have that $\E_{i\in[\pop]} m_i =
\phi$ since $|B|=\phi\cdot \pop.$ Hence,
\[
\query^*(\dist)
= \E_{(i,\sk)\sim\dist} \query^*(i,\sk)
= \phi - \Prob{i \getsr [\pop]}{i\in T^R \land m_i = 1}
\]
where $\prob{i\in T^R \land m_i = 1} \leq \prob{i \in T^{R}}= |T^R|/\pop=(\sample-500)/\pop \le 1/2000.$
\end{proof}

We will show that the oracle cannot guess the value of $\phi$ with sufficiently high 
probability in $\idealgame$ provided that the recovery phase succeeded.
\begin{claim} \label{clm:idealattacksucceeds}
For every polynomial $\sample = \sample(\dimension)$ and every sufficiently large $\dimension \in \N$,
$$
\Prob{\idealgame_{\sample, \dimension}[\oracle]}{\left( |S \setminus T^{R}| \leq 500 \right) \land \left|\oracle(x, \query^{*}) -
\phi\right| \leq \frac{1}{2000}} \leq 1 - \Omega(1)
$$
\end{claim}
\begin{proof}
Consider the case where $\phi=1/500.$  We have
\begin{align}
&\prob{\forall i \in S \setminus T^{R}, \; m_i = 0 \left|\, |S \setminus T^{R}| \leq 500 \right.}
={}\prod_{i=0}^{499} \left( \frac{(1-\phi)\pop - i}{\pop - i} \right) \notag \\
\leq{} &\left(\frac{(1-\phi)\pop - 499}{\pop-499} \right)^{500}
={} \left(1 - \phi \frac{\pop}{\pop - 499} \right)^{500}
\leq{} \left(1 - 2\phi\right)^{500} \tag{$\pop \geq 2000$} \\
={} &\left(1 - \frac{1}{250} \right)^{500} \geq \frac{1}{4e^2} \notag \notag
\end{align}
On the other hand, when $\phi = 0$, we have
\[
\prob{\forall i \in S \setminus T^{R}, \; m_i = 0 \left| |S \setminus T^{R}| \leq 500
\right.} = 1.
\]
Note that in $\idealgame$ the oracle only sees $m_i$ for $i\in S\setminus
T^R.$ 
When the oracle sees only that $m_i = 0$ for every $i \in S \setminus T^{R}$,
it cannot give an answer that is simultaneously accurate to within $1/2000$
for both the case of $\phi = 0$ and for the case of $\phi = 1/500$. The
event $m_i = 0$ for every $i \in S \setminus T^{R}$ occurs with at least
probability $1/2e$ as shown above. Conditioned on this event, both cases $\phi=0$
and $\phi=1/500$ have constant probability. Hence, the answer of the oracle must be
far from $\phi$ with constant probability.
Formally,
\[
\Prob{\idealgame}{\left.\left|\oracle(x, \query^{*}) - \phi\right| \leq
\frac{1}{2000}\,\right|\, |S \setminus T^{R}| \leq 500} \leq 1 - \Omega(1)\,.
\]
By Claim~\ref{clm:idealrecovery}, this implies
\[
\Prob{\idealgame}{\left(\left|\oracle(x, \query^{*}) - \phi\right| \leq
\frac{1}{2000}\right) \land (|S \setminus T^{R}| \leq 500)} \leq 
\frac{1 - \Omega(1)}{1-o(1)}\leq 1-\Omega(1)\,.
\]
\end{proof}

As in the analysis of the recovery phase, we will first argue that the probability the oracle is accurate for $\query^*$ in $\realgame$ is nearly the same as it is in an $\idealgame$ where the query $\query^{*}$ has been modified to contain no information about users outside of $S$.
\begin{claim} \label{clm:indist3}
Let $Z_{3}$ be the event
$\left\{|S \setminus T^{R}| \leq 500 \land \left|\oracle(x, \query^{*}) - \phi\right| \leq
\frac{1}{2000}\right\}.$
Let $(\encgen, \encenc, \encdec)$ is a computationally secure encryption scheme with key length $\keylength(\security) = \security$ and $\sample = \sample(\dimension)$ be any polynomial.  Then if $\oracle$ is computationally efficient,
$$
\left| \Prob{\idealgame_{\sample, \dimension}[\oracle]}{Z_{3}} - \Prob{\realgame_{\sample, \dimension}[\oracle]}{Z_{3}} \right| \leq \negl(\dimension)
$$
\end{claim}
The proof will follow from the security of the encryption scheme.  We defer the details to Section~\ref{sec:securityproofs}.

\begin{lemma} \label{lem:realattacksucceeds}
If $\oracle$ is computationally efficient, then for any polynomial $\sample = \sample(\dimension)$, and every sufficiently large $\dimension \in \N$,
$$
\Prob{\realgame_{\sample,\dimension}[\oracle]}{(|S \setminus T^{R}| \leq 500) \land \left|\oracle(x, \query^{*}) -
\query^{*}(\dist)\right| \leq \frac{1}{2000}} \leq 1 - \Omega(1)
$$
\end{lemma}
\begin{proof}
By Claim~\ref{clm:attackp}, we have that $|\query^*(\dist)-\phi|\le 1/2000.$
Combining Claim~\ref{clm:idealattacksucceeds} with
Claim~\ref{clm:indist3}, we further have
\[
\Prob{\realgame}{(|S \setminus T^{R}| \leq 500) \land \left|\oracle(x, \query^{*}) -
\phi\right| \leq \frac{1}{2000}} \leq 1 - \Omega(1)\,.
\]
The statement of the lemma now follows from a triangle inequality.
\end{proof}

\subsection{Putting it together}

We can now prove Theorem~\ref{thm:main1} from the introduction.

\begin{theorem}
There is no computationally efficient oracle $\oracle$ that is $(1/2000)$-accurate for $\sample^{3+o(1)}$ adaptively chosen queries given $\sample$ samples.
\end{theorem}

\begin{proof}
The entire attack will consist of $\queries = R \cdot \length + 1$ queries, where $R = \sample - 500$ and $\length = \tilde{O}(\pop^2) = \tilde{O}(\sample^2)$.  Therefore the entire attack consists of $\queries = \sample^{3+o(1)}$ queries.  

Therefore, if $\oracle$, by Lemma~\ref{lem:realrecoverysucceeds},
$$
\Prob{\realgame_{\sample, \dimension}[\oracle]}{|S \setminus T^{R}| > 500} = o_{\sample}(1)
$$
By Lemma~\ref{lem:realattacksucceeds},
$$
\Prob{\realgame_{\sample, \dimension}[\oracle]}{(|S \setminus T^{R}| \leq 500) \land \left|\oracle(x, \query^{*}) -
\query^{*}(\dist)\right| \leq \frac{1}{2000}} \leq 1 - \Omega(1)
$$
Combining these two statements gives
$$
\Prob{\realgame_{\sample, \dimension}[\oracle]}{ \left|\oracle(x, \query^{*}) - \query^{*}(\dist)\right| \leq
\frac{1}{2000}} \leq 1 - \Omega(1)
$$
However, the definition of an accurate oracle asserts, in particular, that
$$
\Prob{\realgame_{\sample, \dimension}[\oracle]}{ \left|\oracle(x, \query^{*}) - \query^{*}(\dist)\right| \leq \frac{1}{2000}}
\geq 1-o(1)
$$
and thereby we obtain a contradiction.
\end{proof}

\subsection{An information-theoretic lower bound} \label{sec:infotheoretic}
In this section we show how our argument can be extended to give an information-theoretic lower bound when the number of samples is much smaller than the dimensionality of the data (Theorem~\ref{thm:main2}).  Since the argument follows the outline of the computational hardness result quite closely, we only highlight the parts of the proof that need modification.  The assumption that the oracle is computationally efficient was used only to establish that $\realgame_{\sample, \dimension}[\oracle]$ and $\idealgame_{\sample, \dimension}[\oracle]$ are indistinguishable (Claims~\ref{clm:indist1},~\ref{clm:indist2}, and~\ref{clm:indist3}).  The argument (proven in Section~\ref{sec:securityproofs}) relies on two facts: 1) that no adversary who runs in time $\poly(\security)$ can distinguish between an oracle that returns encryptions of chosen messages $m_1,\dots,m_{\queries}$ and an oracle that returns encryptions of $0$, and 2) there is a secure encryption scheme with keys of length $\keylength(\security) = \security$, therefore if $\sample(\dimension)$ is polynomial, $\realgame_{\sample, \dimension}[\oracle]$ will instantiate the encryption with security parameter $\lambda = \Omega(\dimension)$.

In order to prove an analogous information-theoretic statement, we first observe that the number of messages that will be encrypted during the execution of $\realgame_{\sample, \dimension}[\oracle]$ is at most $\queries = \tilde{O}(\sample^{3})$, and the encryption scheme only needs to be secure for $\queries$ single-bit messages.  For any number of single-bit messages $\queries$, there exists an encryption scheme that is secure for $\queries$ messages with key length $\keylength(\security, \queries) = \queries$ (namely, the classic ``one-time pad'' encryption scheme).  Therefore, if we choose $\dimension \geq \queries + \lceil \log \pop \rceil = \tilde{O}(\sample^{3})$, we can instantiate $\realgame_{\sample, \dimension}[\oracle]$ using the information-theoretically secure encryption scheme.  This will suffice to prove an information-theoretic analogue of Claims~\ref{clm:indist1},~\ref{clm:indist2}, and~\ref{clm:indist3}, which combined with the remaining arguments of Sections~\ref{sec:recoveryanalysis} and~\ref{sec:attackanalysis}, yields Theorem~\ref{thm:main2} in the introduction.  Formally,

\begin{theorem}
There is no oracle $\oracle$ that for every $\sample \in \N$ is $(1/2000, o_{\sample}(1))$-accurate for $\sample^{3+o(1)}$ adaptively chosen queries given $\sample$ samples in $\bits^{\dimension}$ for $\dimension = \dimension(\sample) = \sample^{3+o(1)}$.
\end{theorem}

\section{Lower bounds for avoiding blatant non-privacy}
In this section we show how our arguments also imply that computationally efficient oracles that guarantee accuracy for adaptively chosen statistical queries must be blatantly non-private, and thereby establish Theorem~\ref{thm:main3} in the introduction.

\subsection{Blatant non-privacy}
Before we can define blatant non-privacy, we need to define a notion of accuracy that is more appropriate for the application to privacy.  In contrast to Definition~\ref{def:accurateoracle} where accuracy is defined with respect to the distribution, here we define accurate with respect to the sample itself.  With this change in mind, we model blatant non-privacy via the following game.

\begin{figure}[ht]
\begin{framed}
\begin{algorithmic}
\STATE{$\bnpadv$ chooses a set $y = \{y_1,\dots,y_{2\sample}\} \subseteq \bits^{\dimension}$}
\STATE{Sample a random subset $x \subseteq y$ of size $\sample$}
\STATE{For $j = 1,\dots,\queries$}
\INDSTATE[1]{$\accadv(\query_{1}, a_{1}, \dots, \query_{j-1}, a_{j-1})$ outputs a query $\query_j$}
\INDSTATE[1]{$\oracle(x, \query_{j})$ outputs $a_{j}$}
\STATE{$\bnpadv$ outputs a set $x' \subseteq y$}
\end{algorithmic}
\end{framed}
\vspace{-6mm}
\caption{$\bnpgame_{\sample, \dimension}[\oracle, \bnpadv]$}
\end{figure}

\begin{definition}
An oracle $\oracle$ is \emph{$(\alpha,\beta)$-sample-accurate for $\queries$ adaptively chosen queries given $\sample$ samples in $\bits^{\dimension}$} if for every adversary $\bnpadv$,
$$
\Prob{\bnpgame_{\sample, \dimension, \queries}[\oracle, \bnpadv]}{\forall
j \in [\queries] \; \left| \oracle(x, \query_{j}) - \query_j(x) \right| \leq \alpha}
\geq 1 - \beta\,.
$$
As a shorthand, we will say that $\oracle$ is \emph{$\alpha$-sample-accurate for $\queries$ queries} if for every $\sample, \dimension \in \N$, $\oracle$ is $(\alpha,o_{\sample}(1))$-accurate for $\queries$ queries given $\sample$ samples in $\bits^{\dimension}$.  Here, $\queries$ may depend on $\sample$ and $\dimension$ and $o_{\sample}(1)$ is a function of $\sample$ that tends to $0$.
\end{definition}

\begin{definition}
\emph{Giving $\alpha$-accurate answers to $\queries$ adaptively chosen queries is blatantly non-private for efficient oracles} if there exists an adversary $\bnpadv$ such that for every oracle $\oracle$ that is computationally efficient and $\alpha$-sample-accurate for $\queries$ adaptively chosen queries,
$$
\Prob{\bnpgame_{\sample, \dimension, \queries}[\oracle, \bnpadv]}{|x \triangle x'| >  \sample / 100} \leq o_{\sample}(1)
$$

If the conclusion holds even for computationally inefficient oracles then we replace ``for efficient oracles'' with ``for unbounded oracles'' in the definition.
\end{definition}

\subsection{Lower bounds}
In this section we show the following theorem
\begin{theorem} \label{thm:nonprivacy1}
Giving accurate answers to $\sample^{3+o(1)}$ adaptively chosen queries is blatantly non-private for computationally efficient oracles.
\end{theorem}

We establish this theorem via an adversary that essentially performs only the reconstruction phase of $\realgame$.  The adversary is described in Figure~\ref{fig:bnpattack}.  Observe that in $\realgame$, we have already established that there is an adversary that recovers a set $T^{R}$ such that $|T^{R} \triangle x| \leq 500$ when $x$ is drawn from a distribution $\dist$ and $\oracle$ gives accurate answers for the distribution $\dist$.  The key difference between that guarantee and the one we must establish, is that here we want to establish blatant non-privacy when the oracle is accurate for the sample.  However, this can be addressed via a fairly simple modification to the argument.

\begin{figure}[ht]
\begin{framed}
\begin{algorithmic}
\STATE{The set $y$:}
\INDSTATE[1]{Given parameters $\dimension, \sample$, let $R = .99 \sample$}
\INDSTATE[1]{Let $\security$ be the largest integer such that $\keylength(\lambda) + \lceil \log 2\sample \rceil \leq \dimension$}
\INDSTATE[1]{For $i \in [2\sample]$, let $\sk_{i} \getsr \encgen(1^{\security})$ and let $y_i = (i, \sk_{i})$}
\STATE{}
\STATE{Choose a subsample $x \subseteq y$ of size $\sample$}
\STATE{}
\STATE{Recovery phase:}
\INDSTATE[1]{Setup fingerprinting codes and ciphertexts:}
\INDSTATE[2]{Let $F^{1},\dots,F^{R} \getsr \fpcgen(1^{2\sample})$, let $\length = \length(2\sample)$ be code length}
\INDSTATE[2]{For $r = 1,\dots, R$, $i = 1,\dots,2\sample$, $j = 1,\dots,\length$, let $\ct^{r}(i,j) = \encenc(\sk_i, F^{r}(i,j))$}
\INDSTATE[1]{Let $T^0 = \emptyset$}
\INDSTATE[1]{For round $r = 1,\dots,R$:}
\INDSTATE[2]{For $j = 1,\dots, \length \colon$}
\INDSTATE[3]{Define the query $\query^{r}_{j}(i', \sk')$ to be $\encdec(\sk',\ct^{r}(i', j))$ if $i' \not\in T^{r-1}$ and $0$ otherwise}
\INDSTATE[3]{Let $a^{r}_{j} = \oracle(x; \query^{r}_{j})$}
\INDSTATE[2]{Let $a^{r} = \left(\frac{\sample}{\sample - r}\right)(a^{r}_{1},\dots, a^{r}_{\length})$}
\INDSTATE[2]{Let $i^{r} = \fpctrace(F^{r}, a^{r})$, and let $T^{r} = T^{r-1} \cup \{i^r\}$}
\end{algorithmic}
\end{framed}
\vspace{-6mm}
\caption{$\bnpattack_{\sample, \dimension}[\oracle]$}
\label{fig:bnpattack}
\end{figure}

As before, we introduce a related ``ideal attack'' for which we can show recovery succeeds.

\begin{figure}[ht]
\begin{framed}
\begin{algorithmic}
\STATE{The set $y$:}
\INDSTATE[1]{Given parameters $\dimension, \sample$, let $R = .99\sample$}
\INDSTATE[1]{Let $\security$ be the largest integer such that $\keylength(\lambda) + \lceil \log 2\sample \rceil \leq \dimension$}
\INDSTATE[1]{For $i \in [2\sample]$, let $\sk_{i} \getsr \encgen(1^{\security})$ and let $y_i = (i, \sk_{i})$}
\STATE{}
\STATE{Choose a subsample $x \subseteq y$ of size $\sample$}
\STATE{}
\STATE{Recovery phase:}
\INDSTATE[1]{Setup fingerprinting codes and ciphertexts:}
\INDSTATE[2]{Let $F^{1},\dots,F^{R} \getsr \fpcgen(1^{2\sample})$, let $\length = \length(2\sample)$ be code length}
\INDSTATE[2]{For $r = 1,\dots, R$, {\color{red}$i \in x$}, $j = 1,\dots,\length$, let $\ct^{r}(i,j) = \encenc(\sk_i, F^{r}(i,j))$}
\INDSTATE[2]{{\color{red}For $r = 1,\dots, R$, $i \not\in x$, $j = 1,\dots,\length$, let $\ct^{r}(i,j) = \encenc(\sk_i, 0)$}}
\INDSTATE[1]{Let $T^0 = \emptyset$}
\INDSTATE[1]{For round $r = 1,\dots,R$:}
\INDSTATE[2]{For $j = 1,\dots, \length \colon$}
\INDSTATE[3]{Define the query $\query^{r}_{j}(i', \sk')$ to be $\encdec(\sk',\ct^{r}(i', j))$ if $i' \not\in T^{r-1}$ and $0$ otherwise}
\INDSTATE[3]{Let $a^{r}_{j} = \oracle(x; \query^{r}_{j})$}
\INDSTATE[2]{Let $a^{r} = \left( \frac{\sample}{\sample - r} \right) (a^{r}_{1},\dots, a^{r}_{\length})$}
\INDSTATE[2]{Let $i^{r} = \fpctrace(F^{r}, a^{r})$, and let $T^{r} = T^{r-1} \cup \{i^r\}$}
\end{algorithmic}
\end{framed}
\vspace{-6mm}
\caption{$\idealbnpattack_{\sample, \dimension}[\oracle]$}
\label{fig:idealbnpattack}
\end{figure}

\begin{claim}\label{clm:bnp1}
For every oracle $\oracle$, and every $\sample, \dimension \in \N$
$$
\Prob{\idealbnpattack_{\sample, \dimension}[\oracle]}{\left( \forall j \in [\length], r \in [R] \left| a^{r}_j - \query^{r}_j(x) \right| \leq 1/300 \right) \atop \land \left( |x \setminus T^{R}| > .01\sample \right)} \leq \negl(\sample)
$$
\end{claim}
\begin{proof}
Fix any round $r = 1,\dots,R$, let $U =  x \setminus T^{r-1}$.  We will show that
$$
\Prob{F^{r} \getsr \fpcgen(1^{2\sample})}{ \left(\forall j \in [\length] \, \left| a^{r}_j - \query^{r}_j(x) \right| \leq 1/300 \right) \atop \land \left(i^{r} \not\in U\right) \land \left( \forall r' < r, \, i^{r'} \in x \setminus T^{r'-1} \right)} \leq \negl(\sample)
$$
By the security of the fingerprinting code, we have that for every $U \subseteq [2\sample]$ and every algorithm $\cA$, if $a^{r} \getsr \cA(F_{U})$, then
\begin{equation*}
\Prob{F^r \getsr \fpcgen(1^{2\sample})}{(a^{r} \in
\consistent(F^{r}_{U}) \land (\fpctrace(F^{r}, a^{r}) \not\in U)} \leq  \negl(\sample)
\end{equation*}

Now we need to show that 
\begin{equation} \label{eq:sampleaccuracy1}
\left( \forall r' < r, \, i^{r'} \in x \setminus T^{r'-1} \right) \land (\forall j \in [\length] \, |a_j^{r} - \query^{r}_j(x)| \leq 1/300) \Longrightarrow a^{r} \in \consistent(F^{r}_{U})
\end{equation}
Observe that, by construction
\begin{align}
\query^{r}_{j}(x)
={} &\frac{1}{\sample} \sum_{i \in x} \encdec(\sk_{i}, \ct^{r}_{j}) \notag \\
={} &\frac{1}{\sample} \sum_{i \in U} \encdec(\sk_{i}, \encenc(\sk_{i}, F^{r}(i,j))) + \frac{1}{\sample} \sum_{i \not\in U} \encdec(\sk_{i}, \encenc(\sk_{i}, 0)) \notag \\
={} &\frac{1}{\sample} \sum_{i \in U}  F^{r}(i,j) \label{eq:sampleaccuracy0}
\end{align}
Observe that if
$
\left| a^{r}_{j} - \query^{r}_{j}(x) \right| \leq 1/300
$
and for every $r' < r$, $i^{r'} \in x \setminus T^{r'-1}$, then $|T^{r}| = r$ and $|U| = |x \setminus T^{r}| = \sample - r$.  In this case, we have
\begin{align}
&\left| \left( \frac{\sample}{\sample - r} \right) a^{r}_{j} - \frac{1}{|U|} \sum_{i \in U} F^{r}(i, j) \right| \notag \\
\leq{} &\left| \left( \frac{\sample}{|U|} \right) \query^{r}_{j} - \frac{1}{|U|} \sum_{i \in U} F^{r}(i, j) \right| + \left(\frac{\sample}{\sample - r}\right)\left(\frac{1}{300}\right) \notag \\
={} & \left(\frac{\sample}{\sample - r}\right)\left(\frac{1}{300}\right) \tag{\ref{eq:sampleaccuracy0}} \\
\leq{} &\frac{1}{3} \tag{$r \leq .99\sample$}
\end{align}
So we conclude that if $a^{r}_1,\dots,a^{r}_{\length}$ are accurate to within $1/300$, and for every $r' < r$, $i^{r'} \in x \setminus T^{r'-1}$, then $a^{r} = \left(\frac{\sample}{\sample-r}\right)(a^{r}_1,\dots,a^{r}_{\length})$ is contained in $\consistent(F^{r}_{U})$, which is precisely~\eqref{eq:sampleaccuracy1}.  Combining~\eqref{eq:sampleaccuracy1} with the security of the fingerprinting code, we have
$$
\Prob{F^{r} \getsr \fpcgen(1^{2\sample})}{ \left(\forall j \in [\length] \, \left| a^{r}_j - \query^{r}_j(x) \right| \leq 1/300 \right) \atop \land \left(i^{r} \not\in U\right) \land \left( \forall r' < r, \, i^{r'} \in x \setminus T^{r'-1} \right)} \leq \negl(\sample)
$$
as desired.

To conclude the proof, we observe that
\begin{align*}
&\Prob{\idealbnpattack_{\sample, \dimension}[\oracle]}{\left( \forall j \in [\length], r \in [R] \left| a^{r}_j - \query^{r}_j(x) \right| \leq 1/300 \right) \atop \land \left(|x \setminus T^{R}| > .01\sample \right)} \leq \negl(\sample) \\
={} &\Prob{\idealbnpattack_{\sample, \dimension}[\oracle]}{\left( \forall j \in [\length], r \in [R] \left| a^{r}_j - \query^{r}_j(x) \right| \leq 1/300 \right) \atop \land \left( \forall r \in [R]\; i^{r} \in x \setminus T^{r-1} \right)} \leq \negl(\sample) \\
\leq{} &\Prob{F^{r} \getsr \fpcgen(1^{2\sample})}{ \left(\forall j \in [\length] \, \left| a^{r}_j - \query^{r}_j(x) \right| \leq 1/300 \right) \atop \land \left(i^{r} \not\in U\right) \land \left( \forall r' < r, \, i^{r'} \in x \setminus T^{r'-1} \right)}  \\
\leq{} &R \cdot \negl(\sample) \leq \negl(\sample)
\end{align*}
\end{proof}

As we did in proving the lower bounds for answering adaptively chosen statistical queries, we now claim that $\bnpattack$ and $\idealbnpattack$ are computationally indistinguishable
\begin{claim}\label{clm:bnp2}
Let $Z$ be the event 
$$
\left( \forall j \in [\length], r \in [R] \left| a^{r}_j - q_j(x) \right| \leq 1/300 \right) \land \left(|x \setminus T^{R}| > .01\sample \right)
$$
Let $(\encgen, \encenc, \encdec)$ is a computationally secure encryption scheme with key length $\keylength(\security) = \security$ and $\sample = \sample(\dimension)$ be any polynomial.  Then if $\oracle$ is computationally efficient,
$$
\left| \Prob{\bnpattack_{\sample, \dimension}[\oracle]}{Z} - \Prob{\idealbnpattack_{\sample, \dimension}[\oracle]}{Z}  \right| = \negl(\dimension)
$$
\end{claim}
The analysis is essentially identical to what was shown in the proof of the lower bounds for answering adaptively chosen statistical queries, so we omit the proof.

We can now combine these two claims to prove Theorem~\ref{thm:nonprivacy1}
\begin{proof}[Proof of Theorem~\ref{thm:nonprivacy1}]
By combining Claim~\ref{clm:bnp1} and~\ref{clm:bnp2}, we have
$$
\Prob{\bnpattack_{\sample, \dimension}[\oracle]}{\left( \forall j \in [\length], r \in [R] \left| a^{r}_j - q_j(x) \right| \leq 1/300 \right) \atop \land \left(|x \setminus T^{R}| > .01\sample \right)} \leq \negl(\sample)
$$

If $\oracle$ is $\alpha$-accurate, then
$$
\Prob{\bnpattack_{\sample, \dimension}[\oracle]}{\forall j \in [\length], r \in [R] \left| a^{r}_j - q_j(x) \right| \leq 1/300} \geq 1 -  o_{\sample}(1)
$$
therefore
$$
\Prob{\bnpattack_{\sample, \dimension}[\oracle]}{|x \setminus T^{R}| > .01 \sample} \geq 1 - o_{\sample}(1)
$$
Since $|x \setminus T^{R}| \leq .01\sample$, and $|T^{R}| \leq .99\sample$ implies $|T^{R} \triangle x| \leq .01\sample$, we have
$$
\Prob{\bnpattack_{\sample, \dimension}[\oracle]}{|T^{R} \triangle x| > .01 \sample} \leq o_{\sample}(1)
$$
This completes the proof.
\end{proof}

As we did in Section~\ref{sec:infotheoretic}, we can prove an information-theoretic analogue of our hardness result for avoiding blatant non-privacy.
\begin{theorem}
Giving accurate answers to $\sample^{3+o(1)}$ adaptively chosen queries on $\sample$ samples of dimension $\dimension = \sample^{3+o(1)}$ is blatantly non-private for unbounded oracles.
\end{theorem}
The proof is essentially identical to what is sketched in Section~\ref{sec:infotheoretic}.

\bibliographystyle{alpha}
\bibliography{references}

\appendix

\section{Security reductions from Section~\ref{sec:attack}} \label{sec:securityproofs}
In Sections~\ref{sec:recoveryanalysis} and~\ref{sec:attackanalysis} we made several claims comparing the probability of events in $\realgame$ to the probability of events in $\idealgame$.  Each of these claims follow from the assumed security of the encryption scheme.  In this section we restate and prove these claims.  Since the claims are all of a similar nature, the proof will be somewhat modular.

Before we begin recall the formal definition of security of an encryption scheme.  Security is defined via a pair of oracles $\encoracle_0$ and $\encoracle_1$.  $\encoracle_1(\sk_{1},\dots,\sk_{\pop}, \cdot)$ takes as input the index of a key $i \in [\pop]$ and a message $m$ and returns $\encenc(\sk_{i}, m)$, whereas $\encoracle_{0}(\sk_{1},\dots,\sk_{\pop}, \cdot)$ takes the same input but returns $\encenc(\sk_{i}, 0)$.  The security of the encryption scheme asserts that for randomly chosen secret keys, no computationally efficient adversary can tell whether or not it is interacting with $\encoracle_{0}$ or $\encoracle_{1}$.

\begin{definition}
An encryption scheme $(\encgen, \encenc, \encdec)$ is $\emph{secure}$ if for every polynomial $\pop = \pop(\security)$, and every $\poly(\security)$-time adversary $\encadv$, if $\sk_{1},\dots,\sk_{\pop} \getsr \encgen(1^{\security})$
\begin{equation*}
\left| \Prob{}{\encadv^{\encoracle_{0}(\sk_{1},\dots,\sk_{\pop}, \cdot)} = 1} 
- \Prob{}{\encadv^{\encoracle_{1}(\sk_{1},\dots,\sk_{\pop}, \cdot)} = 1} \right| = \negl(\security)
\end{equation*}
\end{definition}

\begin{claim}[Claim~\ref{clm:indist1} Restated] \label{clm:indist1restated}
Let $Z_{1}$ be the event
$$
\textrm{$\forall r \in [R]$, for $.99 \length$ choices of $j \in [\length]$,}  \left| \oracle(x, \query^{r}_j) - \Ex{i \in S \setminus T^{r-1}}{F^{r}(i,j)} \right| \leq 1/3
$$
Let $(\encgen, \encenc, \encdec)$ is a computationally secure encryption scheme with key length $\keylength(\security) = \security$ and $\sample = \sample(\dimension)$ be any polynomial.  Then if $\oracle$ is computationally efficient,
$$
\left| \Prob{\idealgame_{\sample, \dimension}[\oracle]}{Z_{1}} - \Prob{\realgame_{\sample, \dimension}[\oracle]}{Z_{1}} \right| \leq \negl(\dimension)
$$
\end{claim}

\begin{claim}[Claim~\ref{clm:indist2} Restated] \label{clm:indist2restated}
Let $Z_{2}$ be the event
$
\set{|S \setminus T^{R}| \leq 500}
$
Let $(\encgen, \encenc, \encdec)$ is a computationally secure encryption scheme with key length $\keylength(\security) = \security$ and $\sample = \sample(\dimension)$ be any polynomial.  Then if $\oracle$ is computationally efficient,
$$
\left| \Prob{\idealgame_{\sample, \dimension}[\oracle]}{Z_{2}} - \Prob{\realgame_{\sample, \dimension}[\oracle]}{Z_{2}} \right| \leq \negl(\dimension)
$$
\end{claim}

\begin{claim} [Claim~\ref{clm:indist3} Restated] \label{clm:indist3restated}
Let $Z_{3}$ be the event
$$
(|S \setminus T^{R}| \leq 500) \land \left|\oracle(x, \query^{*}) - \query^{*}(\dist)\right| \leq
\frac{1}{2000}
$$
Let $(\encgen, \encenc, \encdec)$ is a computationally secure encryption scheme with key length $\keylength(\security) = \security$ and $\sample = \sample(\dimension)$ be any polynomial.  Then if $\oracle$ is computationally efficient,
$$
\left| \Prob{\idealgame_{\sample, \dimension}[\oracle]}{Z_{3}} - \Prob{\realgame_{\sample, \dimension}[\oracle]}{Z_{3}} \right| \leq \negl(\dimension)
$$
\end{claim}

To prove each of these claims $c \in \{1,2,3\}$, we construct an adversary $\encadv_{c}$ that will attempt to use $\oracle$ to break the security of the encryption.  We construct $\encadv_{c}$ in such a way that its advantage in breaking the security of encryption is precisely the difference in the probability of the event $Z_{c}$ between $\realgame$ and $\idealgame$, which implies that the difference in probabilities is negligible.  The simulator is given in Figure~\ref{fig:simulator}

\begin{figure}[ht]
\begin{framed}
\begin{algorithmic}
\STATE{Simulate constructing and sampling from $\dist$:}
\INDSTATE[1]{Given parameters $\dimension, \sample$, let $\pop = 2000 \sample$, let $R = \sample-500$} 
\INDSTATE[1]{Let $\security$ be the largest integer such that $\keylength(\lambda) + \lceil \log \pop \rceil \leq \dimension$}
\INDSTATE[1]{Sample users $u_{1},\dots,u_{\sample} \getsr [\pop]$, let $S$ be the set of unique users in the sample}
\INDSTATE[1]{Choose new keys $\sk_{i} \getsr \encgen(1^{\security})$ for $i \in S$}
\INDSTATE[1]{For $i \in \sample$, let $x_{i} = (u_{i}, \sk_{u_{i}})$, let $x = (x_1,\dots,x_{\sample})$}
\STATE{}
\STATE{Simulate the recovery phase:}
\INDSTATE[1]{Setup fingerprinting codes and ciphertexts:}
\INDSTATE[2]{Let $F^{1},\dots,F^{R} \getsr \fpcgen(1^{\pop})$, let $\length = \length(\pop)$ be code length}
\INDSTATE[2]{For $r = 1,\dots, R$, $i = 1,\dots,\pop$, $j = 1,\dots,\length$}
\INDSTATE[3]{If $i \in S$ let $\ct^{r}(i,j) = \encenc(\sk_i, F^{r}(i,j))$, otherwise ask $\encoracle$ for an encryption}
\INDSTATE[3]{of $F^{r}(i,j)$ under key $\ssk_{i}$, that is $\ct^{r}(i,j) = \encoracle_b(\ssk_{1},\dots,\ssk_{\pop}, i, F^{r}(i,j))$}
\INDSTATE[1]{Let $T^0 = \emptyset$.  For round $r = 1,\dots,R$:}
\INDSTATE[2]{For $j = 1,\dots, \length \colon$}
\INDSTATE[3]{Define the query $\query^{r}_{j}(i', \sk')$ to be $\encdec(\sk',\ct^{r}(i', j))$ if $i' \not\in T^{r-1}$ and $0$ otherwise}
\INDSTATE[3]{Let $a^{r}_{j} = \oracle(x; \query^{r}_{j})$}
\INDSTATE[2]{Let $a^{r} = (a^{r}_{1},\dots, a^{r}_{\length})$, let $i^{r} = \fpctrace(F^{r}, a^{r})$, and let $T^{r} = T^{r-1} \cup \{i^r\}$}
\STATE{}
\STATE{Attack phase:}
\INDSTATE[1]{Let $\phi = 0$ with probability $1/2$ and $\phi = 1/500$ with probability $1/2$}
\INDSTATE[1]{Sample a random subset $B\subseteq[\pop]$ of size $\phi\cdot \pop.$}
\INDSTATE[1]{Let $m_i = 1$ for all $i\in B$ and $0$ for all $i\in [\pop]\setminus B$}
\INDSTATE[1]{For each $i \in S$, let $\ct^*_{i} = \encenc(\sk_i, m_i)$, 
for each $i\in[\pop]\setminus S,$ ask $\encoracle$ for an encryption}
\INDSTATE[2]{of $m_i$ under key $\ssk_{i}$, that is $\ct^{*}_{i} = \encoracle_b(\ssk_{1},\dots,\ssk_{\pop}, i, m_i)$}
\INDSTATE[1]{Define the query $\query^{*}(i', \sk')$ to be $\encdec(\sk',\ct^{*}_{i'})$ if $i' \not\in T^{R}$ and $0$ otherwise}
\INDSTATE[1]{Let $a^* = \oracle(x, \query^*)$}
\STATE{}
\STATE{Output $1$ if and only if the event $Z_{c}$ occurs}
\end{algorithmic}
\end{framed}
\vspace{-6mm}
\caption{$\encadv^{\encoracle_{b}(\ssk_{1},\dots,\ssk_{\pop}, \cdot)}_{c, \sample, \dimension}$}
\label{fig:simulator}
\end{figure}

\begin{proof}[Proof of Claims \ref{clm:indist1restated}, \ref{clm:indist2restated}, \ref{clm:indist3restated}]
First, observe that for $c \in \set{1,2,3}$, $\encadv_{c}$ is computationally efficient as long as $\fpcgen, \fpctrace$, and $\oracle$ are all computationally efficient.  Efficiency of $\fpcgen$ and $\fpctrace$ will be satisfied by the construction in Theorem~\ref{thm:robustfpcsexist} and efficiency of $\oracle$ is by assumption of the claim.  Also notice $\encadv$ can determine whether $Z_{c}$ has occurred efficiently.

Now we observe that when the oracle is $\encoracle_1$ (the oracle that takes as input $i$ and $m$ and returns $\encenc(\ssk_{i}, m)$), and $\ssk_{1},\dots,\ssk_{\pop}$ are chosen randomly from $\encgen(1^{\security})$, then the view of the oracle is identical to $\realgame_{\sample, \dimension}[\oracle]$.  Specifically, the oracle holds a random sample of pairs $(i, \sk_{i})$ and is shown queries that are encryptions either under keys it knows or random unknown keys.  Moreover, the messages being encrypted are chosen from the same distribution.  On the other hand, when the oracle is $\encoracle_0$ (the oracle that takes as input $i$ and $\ct$ and returns $\encenc(\ssk_{i}, 0)$), then the view of the oracle is identical to $\realgame_{\sample, \dimension}[\oracle]$.  Thus we have that for $c \in \{1,2,3\}$,
\begin{align*}
&\left| \Prob{\idealgame_{\sample, \dimension}[\oracle]}{Z_{c}} - \Prob{\realgame_{\sample, \dimension}[\oracle]}{Z_{c}} \right| \\
={} &\left| \Prob{\ssk_{1},\dots,\ssk_{\pop} \getsr \encgen(1^{\security})}{\encadv_{c, \sample, \dimension}^{\encoracle_0(\ssk_{1},\dots,\ssk_{\pop},\cdot)} = 1} -  \Prob{\ssk_{1},\dots,\ssk_{\pop} \getsr \encgen(1^{\security})}{\encadv_{c, \sample, \dimension}^{\encoracle_1(\ssk_{1},\dots,\ssk_{\pop},\cdot)} = 1} \right| \\
={} &\negl(\security) = \negl(\dimension)
\end{align*}
The last equality holds because we have chosen $\pop = 2000\sample(\dimension) = \poly(\dimension)$, and therefore we have $\security = \dimension - \lceil \log \pop \rceil = \dimension - O(\log \dimension)$.  This completes the proof of all three claims.
\end{proof}

\end{document}